\documentclass[twoside]{article}

\usepackage[accepted]{aistats2022}
\usepackage{chieh}
\usepackage[round]{natbib}
\begin{document}

%

%

\twocolumn[

\aistatstitle{Deep Layer-wise Networks Have Closed-Form Weights}

\aistatsauthor{ Chieh Wu* \And Aria Masoomi* \And  Arthur Gretton \And Jennifer Dy }

\aistatsaddress{ Northeastern University \And  Northeastern University \And University College London \And Northeastern University} ]

\begin{abstract}
There is currently a debate within the neuroscience community over the likelihood of the brain performing backpropagation (BP). To better mimic the brain, training a network \textit{one layer at a time} with only a "single forward pass" has been proposed as an alternative to bypass BP; we refer to these networks as "layer-wise" networks. We continue the work on layer-wise networks by answering two outstanding questions. First, \textit{do they have a closed-form solution?} Second, \textit{how do we know when to stop adding more layers?} This work proves that the \kme is the closed-form weight that achieves the network global optimum while driving these networks to converge towards a highly desirable kernel for classification; we call it the \textit{Neural Indicator Kernel}. 
\end{abstract}

\section{INTRODUCTION}
Due to the brain-inspired architecture of Multi-layered Perceptrons (MLPs), the relationship between MLPs and our brains has been a topic of significant interest  \citep{zador2019critique,walker2020Deep}. This line of research triggered a debate  \citep{whittington2019theories} around the neural plausibility of backpropagation (BP). While some contend that brains cannot simulate BP \citep{crick1989recent,grossberg1987competitive}, others have proposed counterclaims with a new generation of models \citep{hinton2007backpropagation,Lilli2020Backpropagation,liao2016important,bengio2017stdp,guerguiev2017towards,sacramento2018dendritic,whittington2017approximation}. This debate has inspired the search for alternative optimization strategies beyond BP. To better mimic the brain, learning the network \textit{one layer at a time} over a single forward pass (we call it layer-wise network) has been proposed as a more likely candidate to match existing understandings in neuroscience \citep{ma2019hsic,Pogodin2020KernelizedIB,Oord2018RepresentationLW}. Our theoretical work contributes to this debate by answering two open questions regarding layer-wise networks.
\begin{enumerate}
[noitemsep,topsep=0pt,leftmargin=5mm]
    \item \textit{Do they have a closed-form solution? }
    \item \textit{How do we know when to stop adding more layers? }
\end{enumerate}
Question 1 asks if easily computable and \textit{closed-form weights} can theoretically yield networks equally powerful as traditional MLPs, bypassing both BP and Stochastic Gradient Descent (SGD). This question is answered by characterizing the expressiveness of layer-wise networks using only \textit{"trivially learned weights"}. Currently, the \textit{Universal Approximation Theorem} states that a network can approximate any continuous function \citep{cybenko1989approximation, hornik1991approximation,zhou2020universality}. However, it is not obvious that weights of "layer-wise networks" can be computed closed-form requiring only basic operations, a simplicity constraint inspired by biology. As our contribution, we prove that layer-wise networks can classify \textit{any pattern} with trivially obtainable closed-form weights using only \textit{addition}. Surprisingly, these weights turn out to be the \kme.


Identifying the network depth has been an open question for traditional networks. For layer-wise networks, this question reduces down to "\textit{when should we stop adding layers?}. We posit that additional layers become unnecessary if layer-wise networks exhibit a \textit{limiting behavior} where adding more layers ceases to meaningfully change the network. We proved that this is theoretically possible by using \kme as weights. In fact, we show that these networks could be modeled as a \textit{mathematical sequence} of functions that converge by \textit{intentional design}. Indeed, not only can these networks converge, they can be induced to converge towards a highly desirable kernel for classification; we call it the \textit{Neural Indicator Kernel} (NIK).

\textbf{Related Work. } When earlier work investigated how MLPs might relate to the brain, \citet{crick1989recent} and later \citet{bengio2015towards} both claimed that it was "highly unlikely" that the brain is performing BP. According to their work, while BP requires the errors to flow backward to update the weights, the brain has no backward flow of information. Moreover, for BP to be feasible, each layer requires the precise gradient of a future layer. This implies that the brain would need to obtain the gradients of future layers while having access to them at an earlier layer. 

To make MLPs biologically plausible, new research has proposed that weights can be obtained without future gradients by using only local information \citep{nokland2019training,lindsey2020learning}; this is called \textit{local plasticity}. This idea has been successfully implemented by training the network \textit{one layer at a time} \citep{ma2019hsic, Pogodin2020KernelizedIB,belilovsky2019greedy, nokland2019training, Lwe2019PuttingAE}, paving the groundwork on how layer-wise networks might bypass BP. While layer-wise training is an important advancement, they are still far from becoming a viable model for the brain, with residual theoretical questions left unanswered. Is the computation of gradients \textit{theoretically necessary}, or is there a vastly simpler way to obtain the weights? How do we know the network depth? This theoretical work focuses on answering these questions.



Modeling MLPs with kernels has yielded highly impactful theoretical results. MLPs have been popularly modeled as a Gaussian process (GP) \citep{neal2012bayesian,matthews2018gaussian,lee2017deep}, inspiring a significant amount of research \citep{hayou2019impact,lee2019wide,yang2019scaling,duvenaud2014avoiding}. The Neural Tangent Kernel (NTK) \citep{jacot2018neural} has recently been proposed to describe the dynamics of the network during training \citep{jacot2018neural,arora2019exact}. Our work differs in that these kernel networks are not trained layer-wise.

Instead, our network performs layer-wise training as a composition of kernels and relates closer to work by \citep{fahlman1990cascade, ma2019hsic, Pogodin2020KernelizedIB,belilovsky2019greedy, nokland2019training, Lwe2019PuttingAE, kulkarni2017layer, zhuang2011two, montavon2011kernel, mairal2014convolutional, Cho2009KernelMF}. While these implementations also employed composition of kernel networks, our work differs in two key aspects. First, their networks all require some form of optimization during training (commonly with SGD). In contrast, we focused on \textit{identifying a closed-form solution} to entirely bypass optimization. 
Second, none of their work studied how closed-form weights relates to the network depth. While \citet{duan2020kernel} was able to relate depth to a complexity bound on composition of kernels under a different setting and objective, identifying a simple and closed-form solution was not their goal. Moreover, we further demonstrated that by using our closed-form solution, the network converges to the \textit{Neural Indicator Kernel} (NIK).

\section{MODELING LAYER-WISE NETWORKS}
\textbf{Layer Construction. } Let $X \in \mathbb{R}^{n \times d}$ be a dataset of $n$ samples with $d$ features and let $Y \in \mathbb{R}^{n \times \nclass}$ be its one-hot encoded labels with $\nclass$ classes. The $l^{th}$ layer consists of linear weights $W_l \in \mathbb{R}^{m \times q}$ followed by an activation function $\af:\mathbb{R}^{n \times q} \rightarrow \mathbb{R}^{n \times m}$. We interpret each layer as a function $\fm_l$ parameterized by $W_l$ with an input/output denoted as $R_{l-1} \in \mathbb{R}^{n \times m}$ and $R_{l} \in \mathbb{R}^{n \times m}$ where $R_l = \fm_l(R_{l-1}) = \af(R_{l-1}W_l)$. The entire network $\fm$ is the composition of all layers where $\fm$ where $\fm = \fm_L \circ ... \circ \fm_1$. 

\textbf{Network Objective. } 
Given $x_i, y_i$ as the $i^{th}$ sample and label of the dataset, the network output is used to minimize an empirical risk $(\hsic)$ with a loss function $(\ell)$ with a general objective of
    \begin{equation}
    \underset{\fm}{\min} \hspace{0.3cm} \hsic 
    \coloneqq
    \underset{\fm}{\min} \hspace{0.3cm} \frac{1}{n} \sum_{i=1}^n \ell(\fm(x_i), y_i).
     \label{eq:basic_empirical_risk}
    \end{equation}
As Eq.~(\ref{eq:basic_empirical_risk}), we are structurally identical to conventional MLPs where each layer consists of linear weights and an activation function. Yet, we differ by introducing the composition of the first $l$ layers as $\fm_{l^\circ} = \fm_l \circ ... \circ \fm_1$ where $l \leq L$.  This notation $(\fm_{l^\circ})$ connects the data directly to the $l$th layer output where $R_l = \fm_{l^{\circ}}(X)$ and leads to the \textit{key novelty of our theoretical contribution}. Namely, we propose to optimize Eq.~(\ref{eq:basic_empirical_risk}) layer-wise as a sequence of a growing networks by replacing $\fm$ in \eq{eq:basic_empirical_risk} incrementally with a sequence of functions $\{\fm_{l^\circ}\}_{l=1}^L$. This results in a sequence of empirical risks $\{\hsic_l\}_{l=1}^L$ which we incrementally solve. We refer to $\{\fm_{l^\circ}\}_{l=1}^L$ and $\{\hsic_{l}\}_{l=1}^L$ as the \KS and the \RS. 

Solving Eq.~(\ref{eq:basic_empirical_risk}) as \RS is where we differ from tradition, this approach enables us to easily represent, analyze and optimize "layer-wise networks". In contrast to using BP with SGD, we now can use \textit{"closed-form solutions"} for $W_l$ to construct \textit{Kernel Sequences} that drives the \RS to automatically minimize Eq.~(\ref{eq:basic_empirical_risk}). To visualize the network structure and how they form the sequences, refer to Fig.~\ref{fig:notation}.

\begin{figure*}[t]
\center
    \includegraphics[width=12cm,height=3.7cm]{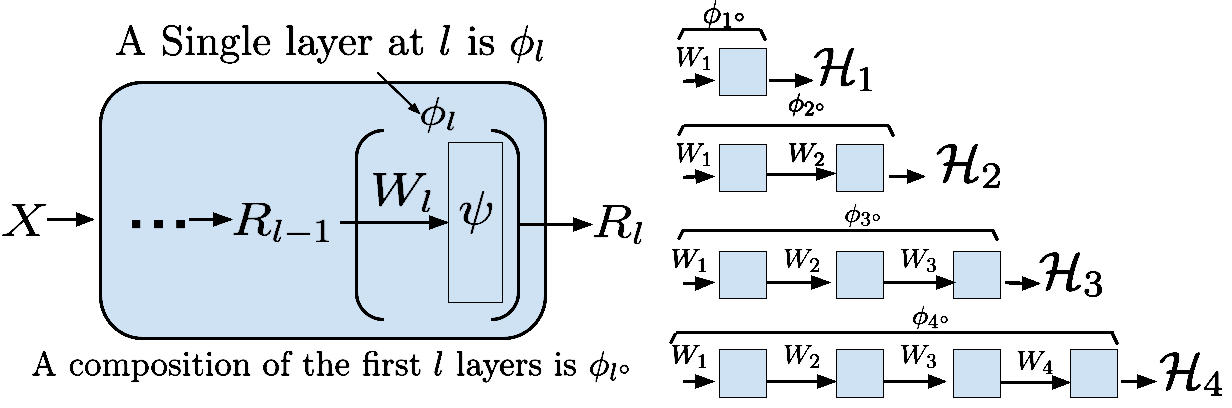}
    \caption{\textbf{Left - Distinguishing } $\phi_l$ vs $\phi_{l^\circ}$\textbf{:} $\fm_l$ is a single layer while $\fm_{l^{\circ}} = \fm_l \circ ... \circ \fm_{1^\circ}$ is a composition of the first $l$ layers.
    \textbf{Right - Visualize the \textit{Kernel} and $\mathcal{H}$-\textit{Sequences}: } Note that the \KS is a converging sequence of "\textit{functions}" $\{\fm_{l^\circ}\}_{l=1} = \{\fm_{1^\circ}, \fm_{2^\circ}, ... \}$ and  \RS is a converging sequence of scalar values $\{\hsic_{l}\}_{l=1} = \{\hsic_{1}, \hsic_{2}, ... \}$.
    To minimize Eq.~(\ref{eq:basic_empirical_risk}) at $\hsic_l$, all weights before $W_l$ are already identified and held fixed, only $W_l$ is unknown. 
    As $l\rightarrow \infty$, the \KS converges to the \textit{Neural Indicator Kernel}.  }
    \label{fig:notation}
\end{figure*} 

\textbf{Performing Classification. } 
Classification tasks typically use objectives like Mean Squared Error ($\mse$) or Cross-Entropy ($\ce$) to match the network output $\fm(X)$ to the label $Y$. While this approach achieves the desirable outcome, it also constrains the space of potential solutions where $\fm(X)$ must match $Y$. Yet, if $\fm$ maps $X$ to the labels $\{0,1\}$ instead of the true label $\{-1,1\}$, 
$\fm(X)$ may not match $Y$, but the solution is the same. Therefore, enforcing $\fm(X) = Y$ ignores an entire space of equivalently optimal classifiers. We posit that by relaxing this constraint and accepting a larger space of potential global optima, it will be easier during optimization to collide with this space. This intuition motivates us to depart from the tradition of label matching and instead seek alternative objectives that focus on solving the underlying prerequisite of classification, i.e.,  learning a mapping where samples from \textit{similar and different} classes become easily distinguishable.

However, since there are many ways to define \textit{similarity}, how do we choose the best one that leads to classification? We demonstrate that the Hilbert Schmidt Independence Criterion (HSIC \citep{gretton2005measuring}) is a highly advantageous objective for this purpose. As we'll later show in corollaries \ref{corollary:mse} and \ref{corollary:ce}, its maximization indirectly minimizes both Mean Square Error (MSE) and Cross-Entropy (CE) under different notions of "distance", enabling classification. Moreover, this is made possible because maximizing HSIC automatically learns the optimal notion of \textit{similarity} as a kernel function. Using HSIC objective as $\hsic_l$ for each element of 
$\{\hsic_{l}\}_{l=1}^L$, the layer-wise formulation of Eq.~(\ref{eq:basic_empirical_risk}) becomes
\begin{equation}
\begin{aligned}
    \max_{W_l} \quad &
        \Tr \left(
        \Gamma \:
        \left[
        \af(R_{l-1}W_l) \af^T(R_{l-1}W_l)
        \right]
        \right)  
        \\
        \st \quad & 
        W_l^TW_l=I,
        \label{eq:main_obj}
\end{aligned}
\end{equation}

where we let $\Gamma = HK_YH = HYY^TH$ with centering matrix $H$ defined as $H = I_n - \frac{1}{n} \mathbf{1}_n \mathbf{1}_n^T$. $I_n$ is an identity matrix of size $n \times n$ and $\textbf{1}_n$ is a column vector of 1s also of length $n$.

Note that the HSIC objective is more familiarly written as $\Tr(HK_YHK_X)$ where $K_X = \af(R_{l-1}W_l) \af^T(R_{l-1}W_l)$. Yet, we purposely present it as Eq.~(\ref{eq:main_obj}) to highlight how the network structure leads to the objective.  Namely, the layer input $R_{l-1}$ first multiplies the weight $W_l$ before passing through the activation function $\af$. The layer output is then multiplied by itself and $\Gamma$ to form the HSIC objective.

Lastly, since HSIC can be trivially maximized by setting each element of $W$ as $\infty$, setting $W^TW=I$ is a constraint commonly used with HSIC while learning a projection \citep{wu2018iterative, Wu2019SolvingIK,niu2010multiple}.

\textbf{Key Difference From Traditional MLPs. } \textit{We generalize the concept of the activation function to a kernel feature map}. Therefore, instead of using the traditional Sigmoid or ReLU activation functions, we use the feature map of a Gaussian kernel as the activation function $\af$. Conveniently, HSIC leverages the kernel trick to spare us the direct computation of the inner product $\af(R_{l-1}W_l) \af^T(R_{l-1}W_l)$. Therefore, for each $(r_i, r_j)$ pair we only need to compute 
\begin{equation}
    \begin{aligned}
    \kf(W_l^T r_i, W^T_lr_j) & = 
    \langle
    \af(W_l^T r_i), \af(W_l^T r_j) 
    \rangle
    \\
    & = 
    \text{exp}\{-
    \frac{||W_l^T r_i - W_l^T r_j||^2}{2\sigma^2_l} \}.
    \label{eq:kernel_def}
    \end{aligned}
\end{equation}

\textbf{How Does HSIC Learn the Kernel? } 
Let $\cS$ be a set of $i,j$ sample pairs that belong to the same class. Its complement, $\cS^c$ contains all sample pairs from different classes. By reinterpreting the kernel $\kf(W_l^T r_i, W^T_lr_j)$ from Eq.~(\ref{eq:kernel_def})  as a kernel function $\kf_{W_l}(r_i, r_j)$ parameterized by $W_l$, \eq{eq:main_obj} can be reformulated into the following objective to see how HSIC learns the kernel.
\begin{equation}
    \begin{aligned}
    \max_{W_l} &
    \sums \Gij \kf_{W_l}(r_i, r_j) 
    -
    \sumsc |\Gij| \kf_{W_l}(r_i, r_j)
    \\
    \st & \quad  W_l^TW_l=I.
    \label{eq:similarity_hsic}
    \end{aligned}
\end{equation}
While \eq{eq:main_obj} and \eq{eq:similarity_hsic} are equivalent objectives, \eq{eq:similarity_hsic} reveals how an optimal similarity measure is learned as a kernel function $\kf_{W_l}(r_i,r_j)$ parameterized by $W_l$.  First note that $\Gamma$ came directly from the label with  $\Gamma=HYY^TH$ such that the $i,j_{th}$ element of $\Gamma$, denoted as $\Gamma_{i,j}$, is a positive value for samples pairs in $\cS$ and negative for $\cS^c$. The objective leverages the sign of $\Gamma_{i,j}$ as labels to guide the choice of $W_l$ such that it increases $\kf_{W_l}(r_i,r_j)$ when $r_i,r_j$ belongs to the same class in $\cS$ while decreasing $\kf_{W_l}(r_i,r_j)$ otherwise. Therefore, by finding a $W_l$ matrix that best parameterizes $\kf_{W_l}$, HSIC identifies the optimal kernel function $\kf_{W_l}(r_i, r_j)$ that separates samples into similar and dissimilar partitions to enable classification.

\section{ANALYZING LAYER-WISE NETWORKS} 
Instead of maximizing Eq.~(\ref{eq:main_obj}) with traditional strategies like SGD, can we completely bypass the optimization step? Our analysis proves that this is possible. In fact, the weights $W_l$ simply need to be set to the \kme \citep{muandet2016kernel} at each layer. The stacking of layers with these known weights automatically drives \RS towards its theoretical global optimum. Specifically, if we let $r^{j}_{\iota}$ be the $\iota^{th}$ input sample in class $j$ for layer $l$ with $\zeta$ as a normalizer that can be ignored in practice, then the closed-form solution is 
   \begin{equation}
        W_s = \frac{1}{\sqrt{\zeta}} \W
        .
        \label{eq:trivial_W}
    \end{equation}
We prove that as long as no two identical samples have conflicting labels, setting $W_l = W_s$ at each layer can generate a monotonically increasing \RS towards the global optimal, $\hsic^*$. Comparing to BP, this finding suggests that instead of using gradients or needing to propagate error backward, a deep classifier can be globally optimized with "only a single forward pass" by \textit{simple addition}. We formally prove this discovery in App.~\ref{app:thm:hsequence} as the following theorem.
\begin{theorem}
\label{thm:hsequence}
From any initial risk $\hsic_0$, there exists a set of bandwidths $\sigma_l$ and a \KS $\{\fm_{l^{\circ}}\}_{l=1}^L$ parameterized by $W_l = W_s$ in \eq{eq:trivial_W} such that:
\begin{enumerate}[topsep=0pt, partopsep=0pt, label=\Roman*.]
    \item 
     $\hsic_L$ can approach arbitrarily close to $\hsic^*$ such that for any $L>1$ and $\delta>0$ we can achieve
     \begin{equation}
        \hsic^{*} - \hsic_L \le \delta,
        \label{arb_close}
    \end{equation}   
    \item 
    as $L \rightarrow \infty$, the \RS converges to the global optimum where 
    \begin{equation}
        \lim_{L \rightarrow \infty}  \hsic_L = \hsic^*,
    \end{equation}   
    \item 
    the convergence is strictly monotonic where 
    \begin{equation}
    \hsic_{l} > \hsic_{l-1} \quad \forall l \ge 1.
    \end{equation}   
\end{enumerate}
\end{theorem}
To summarize the proof, we identified a lower bound for HSIC that is guaranteed to increase at each layer by adjusting $\sigma_l$ while using $W_s$. Since HSIC has a known global theoretical upper bound ($\sums \Gij$), we show that the lower bound can approach arbitrarily close to the upper bound by adding more layers. 

Intuitively, the network incrementally discovers a better feature map to improve the \kme. Imagine classifying a set of dogs/cats, since \kme defines the \textit{average} dog and cat as two points in a high dimensional space. As $L\rightarrow \infty$, \RS pushes these two points maximally apart from each other with representations that enable easy distinction. The continuity of the feature map ensures that dogs are closer to the average dog than the average cat, enabling classification.

Using the \kme has a secondary implication; the layer-wise network is performing \textit{Kernel Density Estimation}  \citep{davis2011remarks}. That is, $W_s$ empowers these networks to predict the likelihood of an outcome without knowing the underlying distribution, relying solely on observations. For classification, multiplying the output of $\fm_{l-1^\circ}$ by its corresponding $W_s$ is equivalent to approximating the probability of a sample being in each class. Moreover, by summarizing the data into the $W_s$, all past examples can be discarded while allowing new examples to update the network by adjusting only the embedding itself, enabling layer-wise networks to self-adapt given new information.


Theoretically, viewing a network as an \RS is the key contribution that enables these interpretations. Its usage raises new questions with tantalizing possibilities in the debate over the brain's likelihood of performing BP. Is the brain more likely computing derivatives with BP, or is it simply repeating \textit{addition}? Can the brain use mechanisms similar to \textit{Kernel Density Estimation} to approximate the likelihood of an event?  This work only presents the theoretical feasibility and does not claim to have the answer.

\textbf{What Does \KS Converge To? } As \RS converges toward $\hsic^*$, how does \KS ultimately lead to classification? This section analyzes the limit behavior of the \KS and explains how samples are geometrically transformed to induce classification as we add layers to the network.

To formalize this relationship, let us define the point before and after the activation function as \textit{preactivation} $R_{l-1}W_l$ and \textit{activation} $\af(R_{l-1}W_l)$ (terminology utilized in NTK~\citep{jacot2018neural}). 
This distinction is also crucial for us because as \KS converges, each stage predictably converges to distinct geometric orientations that enable classification. Keeping this in mind, let us summarize the layer-wise output at \textit{preactivation} using the within $S_w^l$ and between $S_b^l$ class scatter matrices as
\begin{align}
    S_w^l = \sums 
    W_l^T(r_i - r_j)(r_i - r_j)^TW_l
    \\
    S_b^l = \sumsc 
    W_l^T(r_i - r_j)(r_i - r_j)^TW_l.
    \label{eq:bw_scatter}
\end{align}
These matrices are historically important \citep{fisher1936use,mclachlan2004discriminant} because their trace ratio $\sil = \Tr(S_w)/\Tr(S_b)$ 
measures class separability, i.e., a small $\sil$ signifies a clear separation of samples into distinguishable clusters, a crucial criterion for a classifier. Our analysis shows that by maximizing HSIC to achieve $\hsic_l \rightarrow \hsic^*$, it leads to $\sil \rightarrow 0$ as a byproduct, implying a converging behavior where samples from the same class are incrementally clustering into a single point while pushing samples from other classes away. This orientation at convergence renders classification 
at \textit{preactivation} trivial.

While sample pairs at \textit{preactivation} are partitioned into clusters via the "Euclidean distance", they are similarly partitioned via the "angular distance" $\theta$ at \textit{activation}. Indeed, our Thm.~\ref{thm:geometric_interpret} indicates that as $\hsic_l \rightarrow \hsic^*$, sample pairs at \textit{activation} within $\cS$ achieves perfect alignment ($\theta=0$) while pairs in $\cS^c$ become orthogonal to each other, separated by a maximum angle of $\theta=\pi/2$. This implies that as the network layer increases, the inner product between sample pairs from $\cS$ and $\cS^c$ converges to 1 and 0, respectively. This inner product directly defines a kernel at the point of convergence which we call the \textit{Neural Indicator Kernel}. Classification also becomes trivial once a network converges to \kn.

If the \KS converges at layer $L$ such that $\fm_{L^\circ} \approx \fm_{L+1^\circ} \approx \fm_{L+2^\circ} ... $, additional layers stop changing the network as a function, rendering them "\textit{unnecessary}". This finding yields a promising approach to identify network depth by stop adding layers once the objective is sufficiently close to the global optimum.  Therefore, showing that \KS converges in \textit{preactivation} and \textit{activation} is the key to classification and network depth. We highlight this visually observable convergent behavior in Fig. \ref{fig:progression_kdiscovery}, and formally state our results as the following theorem with its detailed proof in App.~\ref{app:thm:geometric_interpret}.


\begin{theorem}
\label{thm:geometric_interpret}
As $l \rightarrow \infty$ and  $\hsic_l \rightarrow \hsic^*$,
\begin{enumerate}[topsep=0pt, partopsep=0pt, label=\Roman*.]
    \item 
    the scatter trace ratio $\mathcal{T}$ approaches 0 where
    \begin{equation}
    \label{eq:scatter_ratio}
    \lim_{l \rightarrow \infty} \frac{\Tr(S_w^{l})}{\Tr(S_b^{l})} =   0
\end{equation}

\item the \KS converges to the following kernel (The Neural Indicator Kernel):
\begin{equation}
    \label{eq:converged_kernel}
    \lim_{l \rightarrow \infty} 
    \kf(x_{i},x_{j})^{l} = 
    \kf^*(x_{i},x_{j}) =  
    \begin{cases} 
    0 \quad \forall i,j \in \cS^c
    \\
    1 \quad \forall i,j \in \cS
    \end{cases}.
\end{equation}
\end{enumerate}
\end{theorem}

As corollaries to Thm.~\ref{thm:geometric_interpret}, the resulting partition of samples under Euclidean and angular distance implicitly satisfies different classification objectives. At \textit{preactivation}, dataset of $\nclass$ classes will be mapped into $\nclass$ distinct points. While these $\nclass$ points may not match the original labels, this difference is inconsequential for classification. In contrast, converged samples at \textit{activation} will reside along $\nclass$ orthogonal axes on a unit sphere. Realigning these results to the standard bases simulates the softmax output. Therefore, as $\hsic_l \rightarrow \hsic^*$, maximizing HSIC leads to a minimization of $\mse$ and $\ce$ without matching the actual labels themselves: instead, \textit{it matches the underlying geometry of how classes are distinguished.}  
We state these findings as two corollaries below with their proofs in App.~\ref{app:corollary:ce}.
\begin{corollary} 
    \label{corollary:mse}
    $\hsic_l \rightarrow \hsic^*$ leads to a minimization of MSE in \textit{preactivation} via a translation of labels.
\end{corollary}
\begin{corollary} 
    \label{corollary:ce}
    $\hsic_l \rightarrow \hsic^*$ leads to a minimization of CE in \textit{activation} via a change of bases.
\end{corollary}


\textbf{Optimal $W^*$ at Each Layer. } 
The analysis of \RS yields a simple $W_s$ that provided us with an interesting neuroscience interpretation. However, \RS can also perform analysis from an optimization perspective. We performed this analysis and show in App.~\ref{app:lemma:W_not_optimal} that $W_s$ is surprisingly not the optimal solution at each layer, or more accurately, we found that $\frac{\partial}{ \partial W_l}\hsic_{l \ne L}(W_s) \ne 0.$
This result suggests that satisfying the layer-wise optimality is not a prerequisite for \RS to eventually converge to the global optimum. Moreover, it also indicates the existence of a potentially superior $W_l^*$ that may even outperform $W_s$. Using a similar derivation proposed by \citet{wu2018iterative, Wu2019SolvingIK}, we set the derivative of Eq.~(\ref{eq:main_obj}) to 0 and found that the optimal $W_l^*$ at each layer is  the most dominant eigenvectors of
\begin{equation}
    \mathcal{Q}_{l^i} = R_{l-1}^T (
    \hat{\Gamma}
 - \text{Diag}(\hat{\Gamma} 1_n))  R_{l-1}
    ,
    \label{eq:phi}
\end{equation}
where $\hat{\Gamma}$ is a function of $W_{l^\mathbf{i}}$ computed with $\hat{\Gamma} = \Gamma \odot K_{R_{l-1}W_{l^{\mathbf{i}}}}$ and the symbol $\odot$ indicates the element-wise product. Unlike $W_s$, this solution guarantees at layer $l$ to achieve  $\frac{\partial}{ \partial W_l}\hsic_l(W^*_l) = 0$. While $W_s$ is a closed-form solution inspired by the brain, \RS also gives us $W^*$ as an optimal solution, highlighting the impressive versatility of \RS as a tool to analyze layer-wise networks.

\section{GENERALIZATION} 
Besides being an optimum solution, $W_l^*$ exhibits many advantages over $W_s$. For example, while $W_s$ experimentally performs well, $W^*$ converges with fewer layers and superior generalization. This raises a well-known question on generalization. It is known that overparameterized MLPs can generalize even without any explicit regularizer \citep{Zhang2017UnderstandingDL}. This observation contradicts classical learning theory and has been a longstanding puzzle \citep{cao2019generalization,brutzkus2017sgd,allen2019learning}. 
Therefore, by being overparameterized with an infinitely wide network, \kn's ability under HSIC to generalize raises similar questions. In both cases, $W_s$ and $W^*$, the HSIC objective employs an infinitely wide network that should result in overfitting.  We ask theoretically, under our framework, what makes HSIC and $W^*$ special? 

Recently, \citet{poggio2020complexity} have proposed that  traditional MLPs generalize  because gradient methods implicitly regularize the normalized weights given an exponential objective (like our HSIC). We discovered a similar impact the process of finding $W^*$ has on HSIC, i.e., HSIC can be reformulated to isolate out $n$ functions $[D_1(W_l), ..., D_n(W_l)]$ that act as a penalty term during optimization. Let $\cS_i$ be the set of samples that belongs to the $i_{th}$ class and let $\cS^c_i$ be its complement, then each function $D_i(W_l)$ is defined as \begin{equation}
\begin{split}
    D_i(W_l) =  &
    \frac{1}{\sigma^2}
    \sum_{j \in \cS_i}
    \Gij \kf_{W_l}(r_i,r_j)
    - \\
    &
    \frac{1}{\sigma^2}
    \sum_{j \in \cS^c_i}
    |\Gij| \kf_{W_l}(r_i,r_j).
    \label{eq:penalty_term}
\end{split}
\end{equation}
Notice that $D_i(W_l)$ is simply \eq{eq:similarity_hsic} for a single sample scaled by $\frac{1}{\sigma^2}$. Therefore, improving $W_l$ also leads to an increase and decrease of $\kf_{W_l}(r_i,r_j)$ associated with $\cS_i$ and $\cS^c_i$ in \eq{eq:penalty_term}, thereby increasing the size of the penalty term  $D_i(W_l)$.  To appreciate how $D_i(W_l)$ penalizes $\hsic$, we propose an equivalent formulation in the theorem below with its derivation in App~\ref{app:thm:regularizer}.
\begin{theorem}
\label{thm:regularizer}
\eq{eq:similarity_hsic} is equivalent to 
\begin{equation}
    \label{eq:generalization_formulation}
\begin{split}
    \max_{W_l} 
    \sum_{i,j} &
    \frac{\Gij}{\sigma^2}
    \ISMexp
    (r_i^TW_lW_l^Tr_j) 
    \\
    & -
    \sum_{i}
    D_i(W_l)
    ||W_l^Tr_i||_2.
\end{split}
\end{equation}
\end{theorem}
Based on Thm.~\ref{thm:regularizer}, $D_i(W_l)$ adds a negative variable cost to the sample norm, $||W_l^Tr_i||_2$, prescribing an implicit regularizer on HSIC. Therefore, as $W_l$ improve HSIC, it also imposes an incrementally heavier penalty on \eq{eq:generalization_formulation}, severely constraining $W_l$.

\section{EXPERIMENTS}
\textbf{Experiment Settings. }
    Our analysis from Thm.~\ref{thm:hsequence} shows that the Gaussian kernel's $\sigma_l$ can be incrementally decreased to guarantee a monotonic increase. Our experiments learn the $\sigma_l$ by incrementally decreasing its value until an improved HSIC is achieved. For results that involves $W^*$, the Gaussian kernel's $\sigma_l$ is set to the value that maximizes HSIC, see App. \ref{app:opt_sigma}. 
    
    The activation function is approximated with RFF of width 300 \citep{rahimi2008random}. The convergence threshold for \RS is set at $\hsic_l > 0.99$. The network structures discovered by $W^*$ for every dataset are recorded and provided in App.~\ref{app:W_dimensions}. The MLPs that use $\mse$ and $\ce$ have weights initialized via the method used in \citet{he2015delving}. All datasets are centered to 0 and scaled to a standard deviation of 1. All sources are written in Python using Numpy, Sklearn and Pytorch \citep{numpy,sklearn_api,paszke2017automatic}, and ran on an Intel Xeon(R) CPU E5-2630 v3 @ 2.40GHz x 16 with 16 total cores.
\begin{figure*}[t]
\center
    \includegraphics[width=14cm,height=3.5cm]{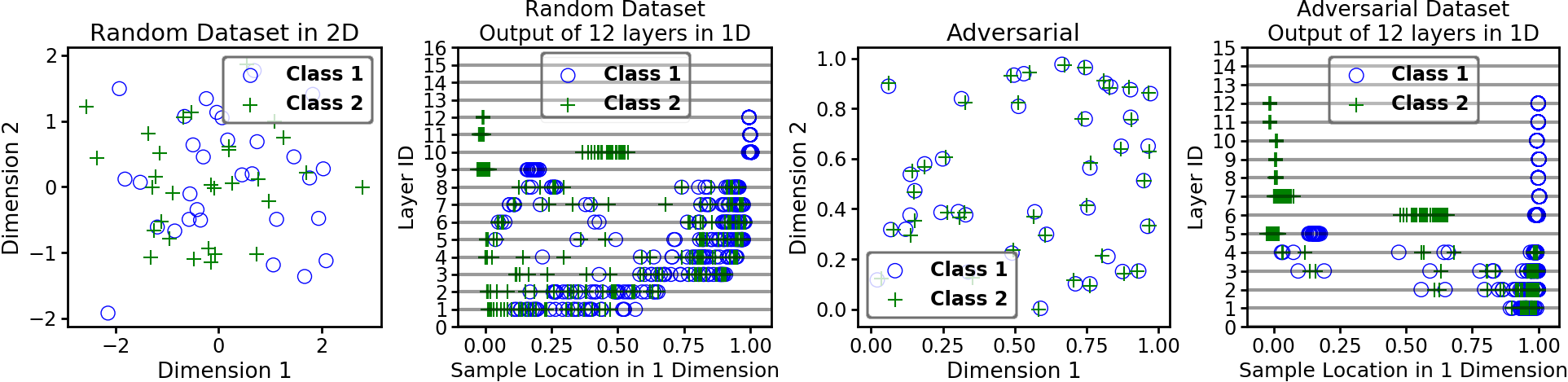}
    \caption{Thm.~\ref{thm:hsequence} is simulated on two highly complex datasets, Random and Adversarial. The 2D representation is shown, and next to it, the 1D output of each layer is displayed over each line. As we add more layers, we can see the samples from the two classes gradually separate from each other. Both datasets achieved the global optimum $\hsic^*$ at the $12_{th}$ layers. For additional results, see App.~\ref{app:sigma_values}}.
    \label{fig:thm_proof}
\end{figure*}

\begin{figure*}[t]
\center
    \includegraphics[width=14cm,height=4cm]{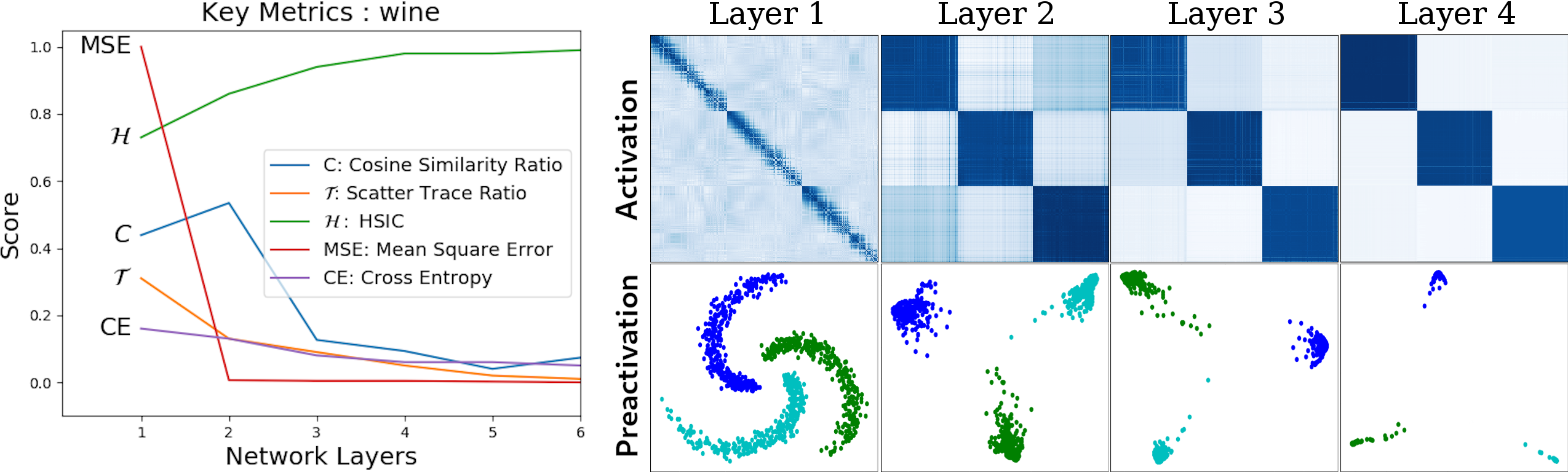}
    \caption{\textbf{Left - Key evaluation metrics at each layer: } Showing key statistics computed from the output of each layer. Notice $\hsic$ monotonically increases while $\sil$ and $\csr$ decrease.  
    \textbf{Right - A visual confirmation of Thm.~\ref{thm:geometric_interpret}: } The top row plots out the kernel matrix from the output of the first 4 layers. Layer 1 shows the original kernel, while layer 4 represents the kernel when it has nearly converged to \kn. The perfect block diagonal structure at layer 4 confirms that the \KS is indeed converging toward \kn. Below the kernels, we also plot out the convergent pattern in \textit{preactivation}. Samples from the same class converge towards a single point while being pushed away from samples from different classes.}
    \label{fig:progression_kdiscovery}
\end{figure*}

\textbf{Datasets. }
The experiments are divided into two sections. Since our contributions are completely theoretical in nature, the majority of the experiments will focus on supporting our thesis by verifying the various components of the theoretical claims. We use three synthetic datasets (Random, Adversarial, and Spiral) and five popular UCI benchmark datasets: wine, cancer, car, face, and divorce ~\citep{Dua:2017}. They are included along with the source code in the supplementary, and their download link and statistics are in App.~\ref{app:data_detail}. All theoretical claims are experimentally reproducible with its source code and datasets publicly available in the Supplementary and at \url{https://github.com/endsley}.

\textbf{Competing Kernel Methods. } We next compare $W^*$ and $W_s$ to several popular kernel networks that have publicly available code: NTK \citep{Arora2020HarnessingTP, Yu2019}, GP \citep{maka892017, Wilson2016DeepKL}, Arc-cos \citep{gionuno2017,Cho2009KernelMF}. 
For this comparison, we additionally included CIFAR10 \citep{Krizhevsky09learningmultiple}.
The images are preprocessed with a convolutional layer that outputs vectorized samples of $x_i \in \mathbb{R}^{10}$. The preprocessing code, the network output, and the network weights are included in the supplementary.


\begin{table*}[t]
\tiny
\centering
\setlength{\tabcolsep}{3.5pt}   
\renewcommand{\arraystretch}{1.1} 
\begin{tabular}{|c|c|c|c|c|c|c|c|c|c|}
	\hline
	 & obj &
		 Train Acc $\uparrow$&
		 Test Acc $\uparrow$&
		 Time(s) $\downarrow$&
		 $\hsic$ $\uparrow$&
		 $\mse$ $\downarrow$&
		 $\ce$ $\downarrow$&
		 $C$ $\downarrow$&
		 $\sil$ $\downarrow$\\ 
		\hline 
	\parbox[t]{2mm}{\multirow{4}{*}{\rotatebox[origin=c]{90}{\textbf{random}}}} &
	 $W^*$ &
		 \textbf{\black{1.00 $\pm$ 0.00}} &
		 \black{0.38 $\pm$ 0.21} &
		 \textbf{\black{0.40 $\pm$ 0.37}} &
		 \textbf{\black{1.00 $\pm$ 0.01}} &
		 \textbf{\black{0.00 $\pm$ 0.01}} &
		 \black{0.05 $\pm$ 0.00} &
		 \textbf{\black{0.00 $\pm$ 0.06}} &
		 \black{0.02 $\pm$ 0.0} \\ 
    & $W_s$ &
    	0.99 $\pm$ 0.01 & 
    	0.45 $\pm$ 0.20 & 
    	0.52 $\pm$ 0.05 & 
    	0.98 $\pm$ 0.01 & 
    	2.37 $\pm$ 1.23 & 
    	0.06 $\pm$ 0.13 & 
    	0.05 $\pm$ 0.02 & 
    	0.13 $\pm$ 0.01\\
	 & $\ce$ &
		 \textbf{\black{1.00 $\pm$ 0.00}} &
		 \black{0.48 $\pm$ 0.17} &
		 \black{25.07 $\pm$ 5.55} &
		 \textbf{\black{1.00 $\pm$ 0.00}} &
		 \black{10.61 $\pm$ 11.52} &
		 \textbf{\black{0.0 $\pm$ 0.0}} &
		 \textbf{\black{0.0 $\pm$ 0.0}} &
		 \textbf{\black{0.0 $\pm$ 0.0}} \\ 
	 & $\mse$ &
		 \black{0.98 $\pm$ 0.04} &
		 \textbf{\black{0.63 $\pm$ 0.21}} &
		 \black{23.58 $\pm$ 8.38} &
		 \black{0.93 $\pm$ 0.12} &
		 \black{0.02 $\pm$ 0.04} &
		 \black{0.74 $\pm$ 0.03} &
		 \black{0.04 $\pm$ 0.04} &
		 \black{0.08 $\pm$ 0.1} \\ 
		\hline 
	\parbox[t]{2mm}{\multirow{4}{*}{\rotatebox[origin=c]{90}{\textbf{Advers}}}} &
    $W^*$ &
		 \textbf{\black{1.00 $\pm$ 0.00}} &
		 \black{0.38 $\pm$ 0.10} &
		 \textbf{\black{0.52 $\pm$ 0.51}} &
		 \textbf{\black{1.00 $\pm$ 0.00}} &
		 \textbf{\black{0.00 $\pm$ 0.00}} &
		 \textbf{\black{0.04 $\pm$ 0.00}} &
		 \textbf{\black{0.01 $\pm$ 0.08}} &
		 \textbf{\black{0.01 $\pm$ 0.0}} \\ 
    & $W_s$ &
    	0.99 $\pm$ 0.04 & 
    	\textbf{0.42 $\pm$ 0.18} & 
    	2.82 $\pm$ 0.81 & 
    	0.99 $\pm$ 0.19 & 
    	15.02 $\pm$ 11.97 & 
    	0.32 $\pm$ 0.15 & 
    	0.30 $\pm$ 0.18 & 
    	0.34 $\pm$ 0.19\\
	 & $\ce$ &
		 \black{0.59 $\pm$ 0.04} &
		 \black{0.29 $\pm$ 0.15} &
		 \black{69.54 $\pm$ 24.14} &
		 \black{0.10 $\pm$ 0.07} &
		 \black{0.65 $\pm$ 0.16} &
		 \black{0.63 $\pm$ 0.04} &
		 \black{0.98 $\pm$ 0.03} &
		 \black{0.92 $\pm$ 0.0} \\ 
	 & $\mse$ &
		 \black{0.56 $\pm$ 0.02} &
		 \black{0.32 $\pm$ 0.20} &
		 \black{113.75 $\pm$ 21.71} &
		 \black{0.02 $\pm$ 0.01} &
		 \black{0.24 $\pm$ 0.01} &
		 \black{0.70 $\pm$ 0.00} &
		 \black{0.99 $\pm$ 0.02} &
		 \black{0.95 $\pm$ 0.0} \\ 
		\hline 
		\hline 
	\parbox[t]{2mm}{\multirow{4}{*}{\rotatebox[origin=c]{90}{\textbf{spiral}}}} &
	 $W^*$ &
		 \textbf{\black{1.00 $\pm$ 0.00}} &
		 \textbf{\black{1.00 $\pm$ 0.00}} &
		 \textbf{\black{0.87 $\pm$ 0.08}} &
		 \black{0.98 $\pm$ 0.01} &
		 \black{0.01 $\pm$ 0.00} &
		 \black{0.02 $\pm$ 0.01} &
		 \black{0.04 $\pm$ 0.03} &
		 \black{0.02 $\pm$ 0.0} \\ 
    & $W_s$ &
    	\textbf{1.00 $\pm$ 0.00} & 
    	\textbf{1.00 $\pm$ 0.00} & 
    	13.54 $\pm$ 5.66 & 
    	0.88 $\pm$ 0.03 & 
    	38.60 $\pm$ 25.24 & 
    	0.06 $\pm$ 0.02 & 
    	0.08 $\pm$ 0.04 & 
    	0.08 $\pm$ 0\\
	 & $\ce$ &
		 \textbf{\black{1.00 $\pm$ 0}} &
		 \textbf{\black{1.00 $\pm$ 0}} &
		 \black{11.59 $\pm$ 5.52} &
		 \textbf{\black{1.00 $\pm$ 0}} &
		 \black{57.08 $\pm$ 31.25} &
		 \textbf{\black{0 $\pm$ 0}} &
		 \textbf{\black{0 $\pm$ 0}} &
		 \textbf{\black{0 $\pm$ 0}} \\ 
	 & $\mse$ &
		 \textbf{\black{1.00 $\pm$ 0}} &
		 \black{0.99 $\pm$ 0.01} &
		 \black{456.77 $\pm$ 78.83} &
		 \textbf{\black{1.00 $\pm$ 0}} &
		 \textbf{\black{0 $\pm$ 0}} &
		 \black{1.11 $\pm$ 0.04} &
		 \black{0.40 $\pm$ 0.01} &
		 \textbf{\black{0 $\pm$ 0}} \\ 
		\hline 
	\parbox[t]{2mm}{\multirow{4}{*}{\rotatebox[origin=c]{90}{\textbf{wine}}}} &
	 $W^*$ &
		 \black{0.99 $\pm$ 0} &
		 \textbf{\black{0.99 $\pm$ 0.05}} &
		 \textbf{\black{0.28 $\pm$ 0.04}} &
		 \black{0.98 $\pm$ 0.01} &
		 \black{0.01 $\pm$ 0} &
		 \black{0.07 $\pm$ 0.01} &
		 \black{0.04 $\pm$ 0.03} &
		 \black{0.02 $\pm$ 0} \\ 
    & $W_s$ &
    	0.98 $\pm$ 0.01 & 
    	0.94 $\pm$ 0.04 & 
    	0.78 $\pm$ 0.09 & 
    	0.93 $\pm$ 0.01 & 
    	2.47 $\pm$ 0.26 & 
    	0.06 $\pm$ 0.01 & 
    	0.05 $\pm$ 0.01 & 
    	0.08 $\pm$ 0.01\\
	 & $\ce$ &
		 \textbf{\black{1.00 $\pm$ 0.00}} &
		 \black{0.94 $\pm$ 0.06} &
		 \black{3.30 $\pm$ 1.24} &
		 \textbf{\black{1.00 $\pm$ 0.00}} &
		 \black{40.33 $\pm$ 35.5} &
		 \textbf{\black{0 $\pm$ 0}} &
		 \textbf{\black{0 $\pm$ 0}} &
		 \textbf{\black{0 $\pm$ 0}} \\ 
	 & $\mse$ &
		 \textbf{\black{1.00 $\pm$ 0}} &
		 \black{0.89 $\pm$ 0.17} &
		 \black{77.45 $\pm$ 45.40} &
		 \textbf{\black{1.00 $\pm$ 0}} &
		 \textbf{\black{0 $\pm$ 0}} &
		 \black{1.15 $\pm$ 0.07} &
		 \black{0.49 $\pm$ 0.02} &
		 \textbf{\black{0 $\pm$ 0}} \\ 
		\hline 
	\parbox[t]{2mm}{\multirow{4}{*}{\rotatebox[origin=c]{90}{\textbf{cancer}}}} &
	 $W^*$ &
		 \black{0.99 $\pm$ 0} &
		 \textbf{\black{0.98 $\pm$ 0.02}} &
		 \textbf{\black{2.58 $\pm$ 1.07}} &
		 \black{0.96 $\pm$ 0.01} &
		 \black{0.02 $\pm$ 0.01} &
		 \black{0.04 $\pm$ 0.01} &
		 \black{0.02 $\pm$ 0.04} &
		 \black{0.04 $\pm$ 0.0} \\ 
    & $W_s$ &
    	0.98 $\pm$ 0.01 & 
    	0.96 $\pm$ 0.03 & 
    	6.21 $\pm$ 0.36 & 
    	0.88 $\pm$ 0.01 & 
    	41.31 $\pm$ 56.17 & 
    	0.09 $\pm$ 0.01 & 
    	0.09 $\pm$ 0.02 & 
    	0.16 $\pm$ 0.03\\
	 & $\ce$ &
		 \textbf{\black{1.00 $\pm$ 0}} &
		 \black{0.97 $\pm$ 0.01} &
		 \black{82.03 $\pm$ 35.15} &
		 \textbf{\black{1.00 $\pm$ 0}} &
		 \black{2330 $\pm$ 2915} &
		 \textbf{\black{0 $\pm$ 0}} &
		 \textbf{\black{0 $\pm$ 0}} &
		 \textbf{\black{0 $\pm$ 0}} \\ 
	 & $\mse$ &
		 \textbf{\black{1.00 $\pm$ 0.00}} &
		 \black{0.97 $\pm$ 0.03} &
		 \black{151.81 $\pm$ 27.27} &
		 \textbf{\black{1.00 $\pm$ 0}} &
		 \textbf{\black{0 $\pm$ 0}} &
		 \black{0.66 $\pm$ 0.06} &
		 \textbf{\black{0 $\pm$ 0}} &
		 \textbf{\black{0 $\pm$ 0}} \\ 
		\hline 
	\parbox[t]{2mm}{\multirow{4}{*}{\rotatebox[origin=c]{90}{\textbf{car}}}} &
	 $W^*$ &
		 \textbf{\black{1.00 $\pm$ 0}} &
		 \textbf{\black{1.00 $\pm$ 0.01}} &
		 \textbf{\black{1.51 $\pm$ 0.35}} &
		 \black{0.99 $\pm$ 0} &
		 \textbf{\black{0 $\pm$ 0}} &
		 \black{0.01 $\pm$ 0.00} &
		 \black{0.04 $\pm$ 0.03} &
		 \black{0.01 $\pm$ 0} \\ 
    & $W_s$ &
    	\textbf{1.00 $\pm$ 0} & 
    	\textbf{1.00 $\pm$ 0} & 
    	5.15 $\pm$ 1.07 & 
    	0.93 $\pm$ 0.02 & 
    	12.89 $\pm$ 2.05 & 
    	0 $\pm$ 0 & 
    	0.06 $\pm$ 0.02 & 
    	0.08 $\pm$ 0.02\\
	 & $\ce$ &
		 \textbf{\black{1.00 $\pm$ 0}} &
		 \textbf{\black{1.00 $\pm$ 0}} &
		 \black{25.79 $\pm$ 18.86} &
		 \textbf{\black{1.00 $\pm$ 0}} &
		 \black{225.11 $\pm$ 253} &
		 \textbf{\black{0 $\pm$ 0}} &
		 \textbf{\black{0 $\pm$ 0}} &
		 \textbf{\black{0 $\pm$ 0}} \\ 
	 & $\mse$ &
		 \textbf{\black{1.00 $\pm$ 0}} &
		 \textbf{\black{1.00 $\pm$ 0}} &
		 \black{504 $\pm$ 116.6} &
		 \textbf{\black{1.00 $\pm$ 0}} &
		 \textbf{\black{0 $\pm$ 0}} &
		 \black{1.12 $\pm$ 0.07} &
		 \black{0.40 $\pm$ 0} &
		 \textbf{\black{0 $\pm$ 0}} \\ 
		\hline 
	\parbox[t]{2mm}{\multirow{4}{*}{\rotatebox[origin=c]{90}{\textbf{face}}}} &
	 $W^*$ &
		 \textbf{\black{1.00 $\pm$ 0}} &
		 \textbf{\black{0.99 $\pm$ 0.01}} &
		 \textbf{\black{0.78 $\pm$ 0.08}} &
		 \black{0.97 $\pm$ 0} &
		 \textbf{\black{0 $\pm$ 0}} &
		 \black{0.17 $\pm$ 0} &
		 \black{0.01 $\pm$ 0} &
		 \textbf{\black{0 $\pm$ 0}} \\ 
    & $W_s$ &
    	0.98 $\pm$ 0.01 & 
    	0.94 $\pm$ 0.26 & 
    	0.86 $\pm$ 0.04 & 
    	3.15 $\pm$ 3.05 & 
    	2.07 $\pm$ 1.04 & 
    	0.28 $\pm$ 0.51 & 
    	0.04 $\pm$ 0.01 & 
    	0.01 $\pm$ 0\\
	 & $\ce$ &
		 \textbf{\black{1.00 $\pm$ 0}} &
		 \black{0.79 $\pm$ 0.31} &
		 \black{23.70 $\pm$ 8.85} &
		 \textbf{\black{1.00 $\pm$ 0}} &
		 \black{16099 $\pm$ 16330} &
		 \textbf{\black{0 $\pm$ 0}} &
		 \textbf{\black{0 $\pm$ 0}} &
		 \textbf{\black{0 $\pm$ 0}} \\ 
	 & $\mse$ &
		 \black{0.92 $\pm$ 0.10} &
		 \black{0.52 $\pm$ 0.26} &
		 \black{745.2 $\pm$ 282} &
		 \black{0.94 $\pm$ 0.07} &
		 \black{0.11 $\pm$ 0.12} &
		 \black{3.50 $\pm$ 0.28} &
		 \black{0.72 $\pm$ 0.01} &
		 \textbf{\black{0 $\pm$ 0}} \\ 
		\hline 
	\parbox[t]{2mm}{\multirow{4}{*}{\rotatebox[origin=c]{90}{\textbf{divorce}}}} &
	 $W^*$ &
		 \black{0.99 $\pm$ 0.01} &
		 \black{0.98 $\pm$ 0.02} &
		 \textbf{\black{0.71 $\pm$ 0.41}} &
		 \black{0.99 $\pm$ 0.01} &
		 \black{0.01 $\pm$ 0.01} &
		 \black{0.03 $\pm$ 0} &
		 \textbf{\black{0 $\pm$ 0.05}} &
		 \black{0.02 $\pm$ 0} \\ 
    & $W_s$ &
    	0.99 $\pm$ 0 & 
    	0.95 $\pm$ 0.06 & 
    	1.54 $\pm$ 0.13 & 
    	0.91 $\pm$ 0.01 & 
    	60.17 $\pm$ 70.64 & 
    	0.04 $\pm$ 0.01 & 
    	0.05 $\pm$ 0.01 & 
    	0.08 $\pm$ 0\\
	 & $\ce$ &
		 \textbf{\black{1.00 $\pm$ 0}} &
		 \textbf{\black{0.99 $\pm$ 0.02}} &
		 \black{2.62 $\pm$ 1.21} &
		 \textbf{\black{1.00 $\pm$ 0}} &
		 \black{14.11 $\pm$ 12.32} &
		 \textbf{\black{0 $\pm$ 0}} &
		 \textbf{\black{0 $\pm$ 0}} &
		 \textbf{\black{0 $\pm$ 0}} \\ 
	 & $\mse$ &
		 \textbf{\black{1.00 $\pm$ 0}} &
		 \black{0.97 $\pm$ 0.03} &
		 \black{47.89 $\pm$ 24.31} &
		 \textbf{\black{1.00 $\pm$ 0}} &
		 \textbf{\black{0 $\pm$ 0}} &
		 \black{0.73 $\pm$ 0.07} &
		 \textbf{\black{0 $\pm$ 0.01}} &
		 \black{0.01 $\pm$ 0} \\ 
	\hline
\end{tabular}
\vspace{3pt}
\caption{Each dataset contains 4 rows comparing the $W^*$ and $W_s$ against traditional MLPs trained using MSE and CE via SGD given the same network width and depth. The best results are in bold with $\uparrow/\downarrow$ indicating larger/smaller values preferred.}
\label{table:main}
\end{table*}
\textbf{Evaluation Metrics.} We report the HSIC objective as $\hsic$ at convergence along with the training/test accuracy for each dataset. Here, $\hsic$ is normalized to the range between 0 to 1 using the method proposed by \citet{cortes2012algorithms}. To corroborate Corollaries~\ref{corollary:mse} and \ref{corollary:ce}, we also record $\mse$ and $\ce$. To evaluate the sample geometry predicted by \eq{eq:scatter_ratio}, we recorded the scatter trace ratio $\sil$ to measure the compactness of samples within and between classes. The angular distance between samples in $\cS$ and $\cS^c$ as predicted by \eq{eq:converged_kernel} is evaluated with the Cosine Similarity Ratio ($\csr$). The equations for normalized $\hsic$ and $\csr$ are
    $\hsic = \frac{\hsic(\fm(X),Y)}{\sqrt{\hsic(\fm(X),\fm(X)) \hsic(Y,Y)}}$ and 
    $\csr = \frac
        {\sum_{i,j \in \mathcal{S}^c} \langle \fm(x_i), \fm(x_j) \rangle}
         {\sum_{i,j \in \mathcal{S}} \langle \fm(x_i), \fm(x_j) \rangle}$.

\textbf{Confirming Theorem 1. }
Since Thm.~\ref{thm:hsequence} guarantees an optimal convergence for \textit{any} dataset given $W_s$, we designed an Adversarial dataset of high complexity, i.e., the sample pairs in $\cS^c$ are intentionally placed significantly closer than samples pairs in $\cS$. We next designed a Random dataset with completely random labels. We then simulated Thm.~\ref{thm:hsequence} in Python and plotted the sample behavior in Fig.~\ref{fig:thm_proof}. The original 2-dimensional data is shown next to its 1-dimensional results: each line represents the 1D output at that layer. As predicted by the theorem, the \RS for both datasets converges to the global optimum at the 12th layer and perfectly separated the samples based on labels. \textit{This experiment simultaneously confirms Thm.~\ref{thm:hsequence} and the idea that simple and repetitive patterns as weights can incrementally improve a network to classify any pattern during training. It also supports the argument that gradients might not be theoretically necessary to optimize deep networks.}

\begin{table*}[t]
\centering
\tiny
\setlength{\tabcolsep}{1.5pt}
\renewcommand{\arraystretch}{1.3}
\begin{tabular}{|c|c|c|c|c|c||c|c|c|c|c|}
	\hline
	 & 	 \multicolumn{5}{|c||}{Training Accuracy} & 
	 \multicolumn{5}{|c|}{Test Accuracy} \\ 
	\hline
	 & 	W* & Ws & GP & Arc-cos & NTK & W* & Ws & GP & Arc-cos & NTK \\ 
	\hline
	\textbf{adversarial} &  
	\textbf{1.00 $\pm$ 0.00} &
	\textbf{1.00 $\pm$ 0.00} &
	0.56 $\pm$ 0.02 &
	0.53 $\pm$ 0.02 &
	0.52 $\pm$ 0.01&
	\textbf{0.53 $\pm$ 0.00} &
	0.50 $\pm$ 0.16 &
	0.17 $\pm$ 0.15 &
	0.25 $\pm$ 0.08 &
	0.30 $\pm$ 0.06
\\ 
	\textbf{random} &  
	\textbf{1.00 $\pm$ 0.00} &
	\textbf{1.00 $\pm$ 0.00} &
	0.95 $\pm$ 0.02 &
	0.66 $\pm$ 0.02 &
	0.63 $\pm$ 0.03&
	0.40 $\pm$ 0.02 &
	0.32 $\pm$ 0.14 &
	\textbf{0.55 $\pm$ 0.22} &
	0.53 $\pm$ 0.21 &
	0.37 $\pm$ 0.21
\\ 
	\textbf{spiral} &  
	\textbf{1.00 $\pm$ 0.00} &
	\textbf{1.00 $\pm$ 0.00} &
	0.99 $\pm$ 0.00 &
	0.99 $\pm$ 0.00 &
	0.99 $\pm$ 0.00&
	\textbf{0.99 $\pm$ 0.01} &
	0.97 $\pm$ 0.00 &
	0.99 $\pm$ 0.01 &
	\textbf{0.99 $\pm$ 0.01} &
	0.98 $\pm$ 0.02
\\ 
	\textbf{wine} &  
	0.99 $\pm$ 0.00 &
	1.00 $\pm$ 0.00 &
	\textbf{1.0 $\pm$ 0.0} &
	\textbf{1.0 $\pm$ 0.0} &
	\textbf{1.0 $\pm$ 0.0} &
	\textbf{0.99 $\pm$ 0.03} &
	0.94 $\pm$ 0.03 &
	0.86 $\pm$ 0.11 &
	0.94 $\pm$ 0.07 &
	0.96 $\pm$ 0.04
\\ 
	\textbf{cancer} &  
	\textbf{0.99 $\pm$ 0.00} &
	\textbf{0.99 $\pm$ 0.00} &
	\textbf{0.99 $\pm$ 0.00} &
	\textbf{0.99 $\pm$ 0.00} &
	0.98 $\pm$ 0.00&
	\textbf{0.98 $\pm$ 0.00} &
	0.98 $\pm$ 0.02 &
	0.97 $\pm$ 0.02 &
	0.97 $\pm$ 0.02 &
	0.97 $\pm$ 0.02
\\ 
	\textbf{car} &  
	\textbf{1 $\pm$ 0.0} &
	\textbf{1 $\pm$ 0.0} &
	\textbf{1 $\pm$ 0.0} &
	\textbf{1 $\pm$ 0.0} &
	\textbf{1 $\pm$ 0.0} &
	\textbf{1 $\pm$ 0.0} &
	0.99 $\pm$ 0.00 &
	0.99 $\pm$ 0.01 &
	\textbf{1.0 $\pm$ 0.0} &
	\textbf{1.0 $\pm$ 0.0} 
\\ 
	\textbf{face} &  
	\textbf{1.0 $\pm$ 0.0} &
	\textbf{1.0 $\pm$ 0.0} &
	0.55 $\pm$ 0.02 &
	\textbf{1.0 $\pm$ 0.0} &
	\textbf{1.0 $\pm$ 0.0} &
	\textbf{1.0 $\pm$ 0.0} &
	0.99 $\pm$ 0.01 &
	0.22 $\pm$ 0.05 &
	0.79 $\pm$ 0.32 &
	0.61 $\pm$ 0.39
\\ 
	\textbf{divorce} &  
	0.99 $\pm$ 0.01 &
	0.99 $\pm$ 0.01 &
	\textbf{1.00 $\pm$ 0.0} &
	\textbf{1.00 $\pm$ 0.0} &
	\textbf{1.00 $\pm$ 0.0} &
	\textbf{0.99 $\pm$ 0.01} &
	0.97 $\pm$ 0.05 &
	0.94 $\pm$ 0.06 &
	0.95 $\pm$ 0.12 &
	0.97 $\pm$ 0.07
\\ 
	\textbf{cifar10} &  
	\textbf{0.99 $\pm$ 0.01} &
	0.81 $\pm$ 0.00 &
	0.77 $\pm$ 0.01 &
	0.94 $\pm$ 0.01 &
	0.94 $\pm$ 0.01 &
	\textbf{0.93 $\pm$ 0.01} &
	0.74 $\pm$ 0.01 &
	0.72 $\pm$ 0.01 &
	\textbf{0.93 $\pm$ 0.01} &
	\textbf{0.93 $\pm$ 0.01}
\\ 
	\hline
\end{tabular}
\caption{Comparing Train and test accuracy between recent kernel networks against $W_s/W^*$. Notice that $W^*$ consistently achieves the highest test Accuracy while $W_s$ performs comparatively to GP. Also, note that $W_s$ and $W^*$ were the only models with a sufficiently large function class to shatter the Adversarial and random dataset. The 10-fold dataset is reshuffled to contrast against Table~\ref{table:main}. }
\label{table:comparison_table}
\end{table*}
 
\textbf{Does $\hsic$ Improve Monotonically? } Thm.~\ref{thm:hsequence} also claims that \RS should be monotonically increasing. Fig.~\ref{fig:progression_kdiscovery} (Left) plots out all key metrics during training \textit{at each layer}. Here, the \RS is clearly monotonic and converging towards a $\hsic^* \approx 1$. Moreover, the trends for $\sil$ and $\csr$ indicate an incremental clustering of samples into separate partitions. Corresponding to low $\sil$ and $\csr$ values, the low MSE and $\ce$ errors at convergence further reinforces the claims of Corollaries~\ref{corollary:mse} and \ref{corollary:ce}. 
These patterns are highly repeatable, as shown across all datasets included in App.~\ref{app:metric_graphs}.

\textbf{Can Layer-wise Networks with $W^*/W_s$ Compete Against MLPs with BP? }  We conduct 10-fold cross-validation spanning 8 datasets and report their mean and the standard deviation for all key metrics in Table~\ref{table:main}. Once our model is trained and has learned its structure, we use the same depth and width to train 2 additional MLPs via SGD, where instead of HSIC, $\mse$ and $\ce$ are used as $\hsic$. \textit{This allows us to compare \kn to traditional networks of identical sizes that are trained by BP/SGD.}

Observing Table~\ref{table:main}, first note that $\mathbf{\hsic} \approx 1$ for both $W_s$ and $W^*$. This corroborates with Thm.~\ref{thm:hsequence} where $\hsic^{*}$ is guaranteed given enough layers. Next, notice that $\boldsymbol{\sil \approx 0}$,  $\boldsymbol{\csr \approx 0}$, \textbf{MSE} $\boldsymbol{\approx 0}$, and \textbf{CE} $\boldsymbol{\approx 0}$. These results are aligned with Thm.~\ref{thm:geometric_interpret} and its corollaries since they indicate that samples of the same class are merging into a single point while minimizing MSE and CE for classification. Moreover, note that MSE/CE failed to shatter the adversarial dataset while $W_s/W^*$ easily handled the data. This suggest that given the same network size, $W_s/W^*$ is significantly more flexible.

Lastly, notice how $W_s$ and $W^*$ perform favorably against traditional MLPs on almost all benchmarks, confirming that layer-wise network weights can be obtained via basic operations to achieve comparable performance. These results affirmatively answer question 1 while validating layer-wise networks as a promising alternative to BP. \textit{Our theoretical and experimental results strongly suggest that layer-wise networks with "closed-form weights" can match MLP performance. }

\textbf{What is the Speed Difference? } The execution time is also included for reference in Table~\ref{table:main}. Since \kn can be obtained via a single forward pass while SGD requires many iterations of backpropagation, \kn \textit{should be faster}. The Time column of Table~\ref{table:main} confirms this expectation by a wide margin. The biggest difference can be observed by comparing the face dataset: $W^*/W_s$ finished in 0.78/0.86 seconds while $\mse$ required 745 seconds, \textit{which is almost a 1000 times difference.}

\textbf{Does $W_s$ and $W^*$ Experimentally Generalize? }
   With the exception of the two random datasets, the Test Accuracy of $W_s$ consistently performed well against MLPs trained using BP/SGD. This suggests that $W_s$ is a trivially obtainable network solution that generalizes.
    From an optimization perspective, $W^*$ impressively generalized even better across all datasets. It further differentiates itself on a high dimension Face dataset where it was the only method that avoided overfitting. While the generalizability of $W_s$ and $W^*$ is still ongoing research, the experimental results are promising.

\textbf{Is the Network Converging to the Neural Indicator Kernel? } 
A visual pattern of \KS converging toward NIK is shown on the right of Fig.~\ref{fig:progression_kdiscovery}. We rearrange the samples of the same class to be adjacent to each other. This allows us to evaluate the kernel quality via its block diagonal structure quickly. Since Gaussian kernels are restricted to values between 0 and 1, we let white and dark blue be 0 and 1 respectively, where the gradients reflect values in between. Our proof predicts that the \KS converges to NIK, evolving from an uninformative kernel into a highly discriminating kernel of perfect \textit{block diagonal structures}. Corresponding to the top row, the bottom row plots out the \textit{preactivation} at each layer. As predicted by Thm.~\ref{thm:geometric_interpret}, the samples of the same class incrementally converge towards a single point. This pattern is consistently observed on all datasets, and the complete collection of the \textit{kernel sequences} for each dataset can be found in App.~\ref{app:kernel_sequence_graph}. 



\textbf{Comparing to Other Deep Kernel Frameworks. }
The development of \RS primarily focused on the biological motivation to model layer-wise networks; it was not designed to automatically outperforming all existing networks. Table~\ref{table:main} satisfies our primary objective by showing that it already performs comparably against traditional MLPs trained with BP. 
For the kernel community, here we supply additional experiments comparing $W_s/W^*$ against several recently proposed kernel networks to demonstrate surprisingly competitive results in Table~\ref{table:comparison_table}. First, note that $W_s$ reliably achieves comparable test accuracy as GP while $W^*$ consistently outperform all kernel networks. This suggests that \RS yields networks that not only generalizes competitively against traditional MLPs, it is also comparable to some recently proposed kernel networks. Next notice for the Training Accuracy, only $W_s$ and $W^*$ were sufficiently expressive to shatter both the adversarial and random dataset, confirming the expressiveness of layer-wise networks. 


\textbf{Conclusion.} 
We have comprehensively tested each theoretical claim, demonstrating how a layer-wise network modeled as an \RS can yield closed-form solutions that perform comparably to MLPs. The convincing results from our experiments strongly align with the predictions made by our theorems, answering the two central questions of this exploratory work. 

Indeed, a repetition of simple rules can incrementally construct powerful networks capable of classifying any pattern, bypassing both BP and SGD. By modeling MLPs as a \KS, it allows us to design convergent behaviors for layer-wise networks to achieve classification while identifying the network depth.

\textbf{Limitations. }
    While \RS can be designed to bypass BP and "aspirationally" mimic the brain, it \textit{cannot and does not claim} that the sequence itself models the brain; it only models layer-wise networks. Although the similarities are compelling, \textit{any} connection relating \RS to the brain is currently premature. We also emphasize that \RS was not proposed to outperform existing commercial networks. It is a theoretical framework intended for \textit{analysis} of layer-wise networks to \textit{explore potential biological connections}. Therefore, theoretically proving the existence of a closed-form solution is the primary contribution. While \RS yields networks that perform comparably to traditional MLPs and recent kernel networks, carefully engineered networks will likely perform better, e.g., ResNet \citep{He2016DeepRL}.  
    
\textbf{Climate Implication. } This work is originally motivated by the desire to reduce computational requirements for deep networks, lowering their carbon footprint. For how this work hopes to help fight climate change, please refer to App.~\ref{app:motivation}.

\textbf{Acknowledgements}
The work described was supported in part by Award Numbers U2COD023375, UH3OD023251, U01 HL089897, U01 HL089856, R01 HL124233, and R01 HL147326 from the NIH, the National Heart, Lung, and Blood Institute, the FDA Center
for Tobacco Products (CTP), and NSF IIS 1546428.
The authors would also like to thank Davin Hill, Chengzhi Shi, Zulqarnain Khan and Kathia Kirschner for helping to review the paper and provide fruitful discussions.

\clearpage



\clearpage
\bibliography{reference}


\clearpage
\appendix

\thispagestyle{empty}

\onecolumn \makesupplementtitle

\begin{appendices} 
\section{How This Work Relates to Climate Change}
\label{app:motivation}    
    Finding an alternative to BP also has significant climate implications.   \citet{strubell2019energy} have shown that some standard AI models can emit over 626,000 pounds of carbon dioxide; a carbon footprint five times greater than the lifetime usage of a car. This level of emission is simply not sustainable in light of our continual explosive growth. Therefore, the environmental impact of BP necessitates a cheaper alternative. Looking at nature, we can be inspired by the brain's learning capability using only a fraction of the energy. Perhaps artificial neurons can also train without the high energy cost to the environment. This is the moral and the foundational motivation for this work in identifying the existence of $W_s$. A closed-form solution holds the potential to significantly reduce the computational requirement and carbon footprint.  Even if our work ultimately failed to mimic the brain, we hope to inspire the community to identify other closed-form solutions and go beyond BP. \\ \\
    This paper aims to promote the discussion of viewing backpropagation alternatives not only as an academic exercise but also as a climate imperative. Yet, this topic is largely ignored by the community. The authors believe the energy costs of training Neural Networks are having a detrimental climate impact and should be an added topic of interest. The earth also needs an advocate, why not us? Therefore, we as a community, must begin addressing how we can ameliorate our own carbon footprint. This exploratory work aims to share a potential path forward for further research that may address these concerns with the community. While $W_s$ is still not ready for commercial usage, we sincerely hope that the community begins to build novel algorithms over our work on \RS and identify a simpler and cheaper path to train our networks. \\
    
\end{appendices} 
\begin{appendices} 
\section{Proof for Theorem \ref{thm:hsequence}}
\label{app:thm:hsequence}

\textbf{Theorem \ref{thm:hsequence}: }
\textit{
For any $\hsic_0$, there exists a set of bandwidths $\sigma_l$ and a \KS $\{\fm_{l^{\circ}}\}_{l=1}^L$ parameterized by $W_l = W_s$ in \eq{eq:trivial_W} such that: }
\begin{enumerate}[topsep=0pt, partopsep=0pt, label=\Roman*.]
    \item 
     $\hsic_L$ can approach arbitrarily close to $\hsic^*$ such that for any $L>1$ and $\delta>0$ we can achieve
     \begin{equation}
        \hsic^{*} - \hsic_L \le \delta,
        \label{arb_close_append}
    \end{equation}   
    \item 
    as $L \rightarrow \infty$, the \RS converges to the global optimum where
    \begin{equation}
        \lim_{L \rightarrow \infty}  \hsic_L = \hsic^*,
    \end{equation}   
    \item 
    the convergence is strictly monotonic where 
    \begin{equation}
    \hsic_{l} > \hsic_{l-1} \quad \forall l \ge 1.
    \end{equation}   
\end{enumerate}

\begin{lemma}
\label{app:lemma:lowerbound}
Given $\sigma_0$ and $\sigma_1$ as the $\sigma$ values from the last layer and the current layer, then there exists a lower bound for $\hsic_l$, denoted as $\lb\Lsigma$ such that

\begin{equation}
    \hsic_l \ge \lb\Lsigma.
\end{equation}
\end{lemma}

\textbf{Basic Background, Assumptions, and Notations. }
\begin{enumerate}
    \item
    The simulation of this theorem for Adversarial and Random data is also publicly available on \url{https://github.com/anonymous}.
    \item  Here we show that this bound can be established given the last 2 layers.  
    \item  $\sigma_0$ is the $\sigma$ value of the previous layer
    \item  $\sigma_1$ is the $\sigma$ value of the current layer
    \item  $\nclass$ is the number of classes
    \item  $n$ is total number of samples
    \item  $n_i$ is number of samples in the $i^{th}$ class
    \item  $\cS$ is a set of all $i,j$ sample pairs where $r_i$ and $r_j$ belong to the same class.
    \item  $\cS^c$ is a set of all $i,j$ sample pairs where $r_i$ and $r_j$ belong to different same classes.
    \item  $\cS^{\beta}$ is a set of all $i,j$ sample pairs that belongs to the same $\beta^{th}$ classes.
    \item  $r_i^{(\alpha)}$ is the $i^{th}$ sample in the $\alpha^{th}$ class among $\nclass$ classes.
    \item  We assume no $r_i \ne r_j$ pair are equal $\forall i \ne j$. 
    \item  
    Among all $r_i \ne r_j$ pairs, there exists an optimal $r_i^*, r_j^*$ pair where  
    $\langle r_i^*, r_j^* \rangle \ge \langle r_i, r_j \rangle$ $\forall r_i \ne r^*_i$ and $r_j \ne r^*_j$. We denote this maximum inner product as 
    \begin{equation}
        \ub = \langle r_i^*, r_j^* \rangle.
    \end{equation}
    \item Here, each $r_i$ sample is assumed to be a sample in the RKHS of the Gaussian kernel, therefore all inner products are bounded such that
    \begin{equation}
        0 \le \langle r_i, r_j \rangle 
        \le
        \ub.
    \end{equation}
    \item We let $W$ be 
    \begin{equation}
        W_s = \frac{1}{\sqrt{\zeta}} \W.
    \end{equation}
    Instead of using an optimal $W^*$ defined as $W^{*} = \argmax_{W} H_{l}(W)$, we use a suboptimal $W_s$ where each dimension is simply the average direction of each class: $\frac{1}{\sqrt{\zeta}}$ is a unnecessary normalizing constant $\zeta = ||W_{s}||_{2}^{2}$. By using $W_s$, this implies that the $\hsic$ we obtain is already a lower bound compare $\hsic$ obtained by $W^*$. But, we will use this suboptimal $W_s$ to identify an even lower bound. Note that based on the definition $W^{*}$, we have the property $\hsic(W^{*}) \geq \hsic(W) \, \forall W$.

    \item  We note that the objective $\hsic$ is
    \begin{equation}
        \hsic = 
        \underbrace{
        \sums \Gij \ISMexpC{}{}}_{\within}
        -
        \underbrace{
        \sumsc |\Gij| \ISMexpC{}{}}_{\between}
    \end{equation}
    where we let $\within$ be the summation of terms associated with the within cluster pairs, and let $\between$ be the summation of terms associated with the between cluster pairs.
\end{enumerate}
    
    
\begin{proof} $\hspace{1pt}$
\begin{adjustwidth}{0.5cm}{0.0cm}
    The equation is further divided into smaller parts organized into multiple sections.
    
    \textbf{For sample pairs in $\cS$. } The first portion of the function can be split into multiple classes where
    \begin{equation}
    \within = 
    \underbrace{
    \sum_{\cS^1} \Gij \ISMexpC{(1)}{(1)}
    }_\text{$\within_1$}
    + ... + 
    \underbrace{
    \sum_{\cS^{\nclass}} \Gij \ISMexpC{(\nclass)}{(\nclass)}}_\text{$\within_{\nclass}$}
    \end{equation}
    Realize that to find the lower bound, we need to determine the minimum possible value of each term which translates to \textbf{maximum} possible value of each exponent. Without of loss of generality we can find the lower bound for one term and generalize its results to other terms due to their similarity. Let us focus on the numerator of the exponent from $\within_1$. Given $W_s$ as $W$, our goal is identify the \textbf{maximum} possible value for
    \begin{equation}
        \underbrace{
        \rij{1}{1}^TW}_{\Pi_1}
        \underbrace{
        W^T\rij{1}{1}}_{\Pi_2}. 
    \end{equation}
    Zoom in further by looking only at $\Pi_1$, we have 
    the following relationships
    \begin{equation}
       \Pi_1 = 
        \underbrace{
            r_i^{(1)^T} W}_{\xi_1}
        - 
        \underbrace{r_j^{(1)^T} W}_{\xi_2} 
    \end{equation}
    \begin{align}
        \xi_1 & = \frac{1}{\sqrt{\zeta}} r_i^{(1)^T} \W \\
        & = \frac{1}{\sqrt{\zeta}}
            r_i^{(1)^T} 
            \begin{bmatrix}
                (r_1^{(1)} + ... + r_{n_1}^{(1)}) &
                ... &
                (r_1^{(\nclass)} + ... + r_{n_{\nclass}}^{(\nclass)})
            \end{bmatrix}
    \end{align}
    \begin{align}
        \xi_2 & = 
        \frac{1}{\sqrt{\zeta}}
        r_j^{(1)^T} \W \\
        & = 
        \frac{1}{\sqrt{\zeta}}
        r_j^{(1)^T} 
            \begin{bmatrix}
                (r_1^{(1)} + ... + r_{n_1}^{(1)}) &
                ... &
                (r_1^{(\nclass)} + ... + r_{n_{\nclass}}^{(\nclass)})
            \end{bmatrix}
    \end{align}   
    By knowing that the inner product is constrained between $[0,\ub]$, we know the maximum possible value for $\xi_1$ and the minimum possible value for $\xi_2$ to be
     \begin{align}
        \xi_1 & = 
        \frac{1}{\sqrt{\zeta}}
            \begin{bmatrix}
                1 + (n_1 - 1)\ub &
                n_2 \ub &
                n_3 \ub &
                ... &
                n_{\nclass} \ub
            \end{bmatrix} \\
         \xi_2 & = 
         \frac{1}{\sqrt{\zeta}}
            \begin{bmatrix}
                1 &
                0 &
                0 &
                ... &
                0
            \end{bmatrix}. 
    \end{align}   
    Which leads to  
    \begin{equation}
        \Pi_1 = 
        \frac{1}{\sqrt{\zeta}}
        (\xi_1 - \xi_2) = 
        \frac{1}{\sqrt{\zeta}}
            \begin{bmatrix}
                (n_1 - 1)\ub &
                n_2 \ub &
                n_3 \ub &
                ... &
                n_{\nclass} \ub
            \end{bmatrix}
    \end{equation}
    Since $\Pi_2^{T}  = \Pi_1$ we have 
    \begin{align}
        \Pi_1 \Pi_2 & = 
            \frac{1}{\zeta}
            [
            (n_1 - 1)^2\ub^2 + 
            n_2^2 \ub^2 +
            n_3^2 \ub^2 +
                ... +
            n_{\nclass}^2 \ub^2] \\
        & = 
            \frac{1}{\zeta}
            [(n_1 - 1)^2+ 
            n_2^2 +
            n_3^2 +
                ... +
            n_{\nclass}^2] \ub^2
    \end{align}
    The lower bound for just the $\within_1$ term emerges as 
    \begin{equation}
        \within_1 \ge 
        \sum_{\cS^1} \Gij e^{
        - \frac{ 
            [(n_1 - 1)^2+ 
            n_2^2 +
            n_3^2 +
                ... +
            n_{\nclass}^2] \ub^2 }{2 \zeta \sigma_1^2}}. 
    \end{equation}
    To further condense the notation, we define the following constant
    \begin{equation}
        \mathscr{N}_g = 
            \frac{1}{2 \zeta}
            [n_1^2+ 
            n_2^2 +
            ... +
            (n_g - 1)^2 + 
            ... + 
            n_{\tau}^2].
    \end{equation}
    Therefore, the lower bound for $\within_1$ can be simplified as 
    \begin{equation}
        \within_1 \ge \sum_{\cS^1} \Gij e^{-\frac{\mathscr{N}_1 \ub^2}{\sigma_1^2}}
    \end{equation}
    and the general pattern for any $\within_g$ becomes
    \begin{equation}
         \within_g \ge \sum_{\cS^i} \Gij e^{-\frac{\mathscr{N}_g \ub^2}{\sigma_1^2}}. 
    \end{equation}
    The lower bound for the entire set of $\cS$ then becomes
    \begin{equation}
        \sums \Gij \ISMexpC{}{} = 
        \within_1 + ... + \within_{\nclass} \ge 
        \underbrace{
        \sum_{g=1}^{\nclass} \sum_{\cS^g} \Gij
        e^{-\frac{\mathscr{N}_g \ub^2}{\sigma_1^2}}}_{\text{Lower bound}}.
    \end{equation}
    
    \textbf{For sample pairs in $\cS^c$. } 
    To simplify the notation, we note that
    \begin{align}
        -\between_{g_1,g_2} & = 
            -
            \sum_{i \in \cS^{g_1}} 
            \sum_{j \in \cS^{g_2}}
            |\Gij| \ISMexpC{(g_1)}{(g_2)} \\
            & = -
            \sum_{i \in \cS^{g_1}} 
            \sum_{j \in \cS^{g_2}}
            |\Gij| e^{-\frac
            {\Tr(W^T (\rij{g_1}{g_1})(\rij{g_1}{g_2})^T W)}
            {2 \sigma_1^2}} \\                
            & = -
            \sum_{i \in \cS^{g_1}} 
            \sum_{j \in \cS^{g_2}}
            |\Gij| e^{-\frac{\Tr(W^T A_{i,j}^{(g_1, g_2)} W)}
            {2 \sigma_1^2}} \\           
    \end{align}
    We now derived the lower bound for the sample pairs in $\cS^c$. We start by writing out the entire summation sequence for $\between$.  
    \begin{equation}
    \begin{split}
    \between & = 
    -\underbrace{
    \sum_{i \in \cS^1} 
    \sum_{j \in \cS^2}
    |\Gij| 
    e^{-\frac{\Tr(W^T A_{i,j}^{(1, 2)} W)}
    {2 \sigma_1^2}}
    }_\text{$\between_{1,2}$} 
    - \underbrace{...}_{\between_{g_1 \ne g_2}}
    -\underbrace{
    \sum_{i \in \cS^1} 
    \sum_{j \in \cS^{\nclass}}
    |\Gij| 
    e^{-\frac{\Tr(W^T A_{i,j}^{(1, \nclass)} W)}
    {2 \sigma_1^2}}
    }_\text{$\between_{1,\nclass}$} 
    \\
    & 
    -\underbrace{
    \sum_{i \in \cS^2} 
    \sum_{j \in \cS^1}
    |\Gij| 
    e^{-\frac{\Tr(W^T A_{i,j}^{(2, 1)} W)}
    {2 \sigma_1^2}}
    }_\text{$\between_{2,1}$} 
    - \underbrace{...}_{\between_{g_1 \ne g_2}}
    -\underbrace{
    \sum_{i \in \cS^2} 
    \sum_{j \in \cS^{\nclass}}
    |\Gij| 
    e^{-\frac{\Tr(W^T A_{i,j}^{(2, \nclass)} W)}
    {2 \sigma_1^2}}
    }_\text{$\between_{2,\nclass}$} 
    \\   
    & ... \\
    & 
    -\underbrace{
    \sum_{i \in \cS^\nclass} 
    \sum_{j \in \cS^1}
    |\Gij| 
    e^{-\frac{\Tr(W^T A_{i,j}^{(\nclass, 1)} W)}
    {2 \sigma_1^2}}
    }_\text{$\between_{\nclass,1}$} 
    - \underbrace{...}_{\between_{g_1 \ne g_2}}
    -\underbrace{
    \sum_{i \in \cS^{\nclass-1}} 
    \sum_{j \in \cS^{\nclass}}
    |\Gij| 
    e^{-\frac{\Tr(W^T A_{i,j}^{(\nclass-1, \nclass)} W)}
    {2 \sigma_1^2}}
    }_\text{$\between_{\nclass-1,\nclass}$} 
    \end{split}
    \end{equation}
    
    Using a similar approach with the terms from $\within$, note that $\between$ is a negative value, so we need to maximize this term to obtain a lower bound. Consequently, the key is to determine the \textbf{minimal} possible values for each exponent term. Since every one of them will behave very similarly, we can simply look at the numerator of the exponent from $\between_{1,2}$ and then arrive to a more general conclusion. Given $W_s$ as $W$, our goal is to identify the \textbf{minimal} possible value for
    \begin{equation}
        \underbrace{
        \rij{1}{2}^TW}_{\Pi_1}
        \underbrace{
        W^T\rij{1}{2}}_{\Pi_2}. 
    \end{equation}
  
    Zoom in further by looking only at $\Pi_1$, we have 
    the following relationships
    \begin{equation}
       \Pi_1 = 
        \underbrace{
            r_i^{(1)^T} W}_{\xi_1}
        - 
        \underbrace{r_j^{(2)^T} W}_{\xi_2} 
    \end{equation}
    \begin{align}
        \xi_1 & = 
        \frac{1}{\sqrt{\zeta}}
        r_i^{(1)^T} \W \\
        & = 
        \frac{1}{\sqrt{\zeta}}
        r_i^{(1)^T} 
            \begin{bmatrix}
                (r_1^{(1)} + ... + r_{n_1}^{(1)}) &
                ... &
                (r_1^{(\nclass)} + ... + r_{n_{\nclass}}^{(\nclass)})
            \end{bmatrix}
    \end{align}
    \begin{align}
        \xi_2 & = 
        \frac{1}{\sqrt{\zeta}}
        r_j^{(2)^T} \W \\
        & = 
        \frac{1}{\sqrt{\zeta}}
        r_j^{(2)^T} 
            \begin{bmatrix}
                (r_1^{(1)} + ... + r_{n_1}^{(1)}) &
                ... &
                (r_1^{(\nclass)} + ... + r_{n_{\nclass}}^{(\nclass)})
            \end{bmatrix}
    \end{align}   
    By knowing that the inner product is constrained between $[0,\ub]$, we know the \textbf{minimum} possible value for $\xi_1$ and the \textbf{maximum} possible value for $\xi_2$ to be
     \begin{align}
        \xi_1 & = 
            \frac{1}{\sqrt{\zeta}}
            \begin{bmatrix}
                1 &
                0 &
                0 &
                ... &
                0
            \end{bmatrix} \\
         \xi_2 & = 
            \frac{1}{\sqrt{\zeta}}
            \begin{bmatrix}
                n_1\ub &
                1 + (n_2 - 1) \ub &
                n_3 \ub &
                ... &
                n_{\nclass} \ub
            \end{bmatrix} 
    \end{align}   
    Which leads to  
    \begin{equation}
        \Pi_1 = 
        \frac{1}{\sqrt{\zeta}}
        (\xi_1 - \xi_2) = 
            \frac{1}{\sqrt{\zeta}}
            \begin{bmatrix}
                1 - n_1\ub &
                -(1 + (n_2 - 1) \ub) &
                -n_3 \ub &
                ... &
                -n_{\nclass} \ub
            \end{bmatrix}
    \end{equation}
    Since $\Pi_2^{T}  = \Pi_1$ we have     
    \begin{align}
        \Pi_1 \Pi_2 & = 
            \frac{1}{\zeta}
            [
            (1- n_1 \ub)^2 + 
            (1 + (n_2-1)\ub)^2 + 
            n_3^2 \ub^2 +
                ... +
            n_{\nclass}^2 \ub^2].
    \end{align}
    The lower bound for just the $\between_{1,2}$ term emerges as 
    \begin{equation}
        -\between_{1,2} \ge 
        - \sum_{\cS^1} \sum_{\cS^2} |\Gij| e^{
        - \frac{ 
            (1- n_1\ub)^2 + 
            (1 + (n_2-1)\ub)^2 + 
            n_3^2 \ub^2 +
                ... +
            n_{\nclass}^2 \ub^2}
            {2 \zeta \sigma_1^2}}. 
    \end{equation}
    To further condense the notation, we define the following function
    \begin{equation}
    \begin{split}
        \mathscr{N}_{g_1, g_2}(\ub) & = 
            \frac{1}{2 \zeta}
            [ n_1^2 \ub^2 + n_2^2 \ub^2 + ... \\ 
            & +  (1- n_{g_1}\ub)^2 + ... + 
            (1 + (n_{g_2}-1)\ub)^2 \\
            & + ...  + n_{\tau}^2 \ub^2].
    \end{split}
    \end{equation}
    Note that while for $\cS$, the $\ub$ term can be separated out. But here, we cannot, and therefore $\mathscr{N}$ here must be a function of $\ub$.  Therefore, the lower bound for $\between_{1,2}$ can be simplified into
    \begin{equation}
        - \between_{1,2} \ge - 
        \sum_{\cS^1} 
        \sum_{\cS^2}
        |\Gij| e^{-\frac{\mathscr{N}_{1,2}(\ub)}{\sigma_1^2}}
    \end{equation}
    and the general pattern for any $\between_{g_1, g_2}$ becomes
    \begin{equation}
         -\between_{g_1,g_2} \ge -
         \sum_{\cS^{g1}} 
        \sum_{\cS^{g2}}
         \Gij e^{-\frac{\mathscr{N}_{g_1,g_2}(\ub)}
         {\sigma_1^2}}. 
    \end{equation}    
    The lower bound for the entire set of $\cS^c$ then becomes
    \begin{align}
        -\sumsc |\Gij| \ISMexpC{}{} & = 
        - \between_{1,2} 
        - \between_{1,3} 
        - ... - 
        \between_{\nclass-1, \nclass} \\
        & \ge 
        - \underbrace{
        \sum_{g_1 \ne g_2}^{\nclass}
        \sum_{i \in \cS^{g_1}} 
        \sum_{j \in \cS^{g_2}} 
        |\Gij|
        e^{-\frac{\mathscr{N}_{g_1,g_2} (\ub)}{\sigma_1^2}}}_{\text{Lower bound}}.
    \end{align}    
    
    \textbf{Putting $\cS$ and $\cS^c$ Together. } 
    \begin{align}
        \hsic & = \within + \between    \\
        & \ge
        \underbrace{
        \sum_{g=1}^{\nclass} \sum_{\cS^g} \Gij
        e^{-\frac{\mathscr{N}_g \ub^2}{\sigma_1^2}}}_{\text{Lower bound of } \within}
        - \underbrace{
        \sum_{g_1 \ne g_2}^{\nclass} 
        \sum_{i \in \cS^{g_1}} 
        \sum_{j \in \cS^{g_2}} 
        |\Gij|
        e^{-\frac{\mathscr{N}_{g_1,g_2} (\ub)}{\sigma_1^2}}}_{\text{Lower bound of } \between}.
    \end{align}
    Therefore, we have identified a lower bound that is a function of $\sigma_0$ and $\sigma_1$ where
    \begin{equation}
        \lb\Lsigma = 
        \sum_{g=1}^{\tau} \sum_{\cS^g} \Gij
        e^{-\frac{\mathscr{N}_g \ub^2}{\sigma_1^2}}
        - 
        \sum_{g_1 \ne g_2}^{\tau} 
        \sum_{i \in \cS^{g_1}} 
        \sum_{j \in \cS^{g_2}} 
        |\Gij|
        e^{-\frac{\mathscr{N}_{g_1,g_2} (\ub)}{\sigma_1^2}}.
    \end{equation}
    From the lower bound, it is obvious why it is a function of $\sigma_1$. The lower bound is also a function of $\sigma_{0}$ because $\ub$ is actually a function of $\sigma_0$. To specifically clarify this point, we have the next lemma. 
\end{adjustwidth}
\end{proof}
\hspace{1pt}

\begin{lemma}
\label{app:lemma:ub_goes_to_zero}
The $\ub$ used in \citelemma{app:lemma:lowerbound} is a function of $\sigma_0$ where $\ub$ approaches to zero as $\sigma_{0}$ approaches to zero, i.e.
\begin{equation}
    \lim_{\sigma_0 \rightarrow 0}  
    \ub = 0.
\end{equation}
\end{lemma}

\textbf{Assumptions and Notations. }
\begin{enumerate}
    \item  
        We use Fig.~\ref{app:fig:two_layer_img} to help clarify the notations. We here only look at the last 2 layers.  
    \item  
        We let $\hsic_0$ be the $\hsic$ of the last layer, and $\hsic_1$, the $\hsic$ of the current layer.
    \item
        The input of the data is $X$ with each sample as $x_i$, and the output of the previous layer are denoted as $r_i$. $\psi_{\sigma_0}$ is the feature map of the previous layer using $\sigma_0$ and $\psi_{\sigma_1}$ corresponds to the current layer. 
        \begin{figure}[h]
        \center
            \includegraphics[width=7cm]{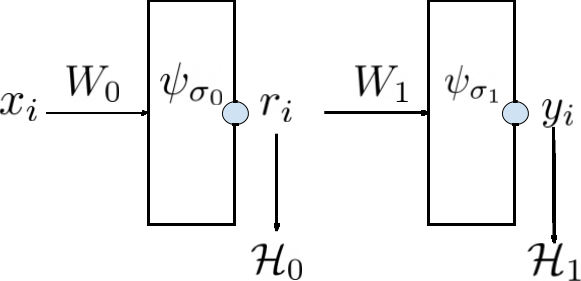}
            \caption{Figure of a 2 layer network.}
            \label{app:fig:two_layer_img}
        \end{figure}     
    \item  
        As defined from \citelemma{app:lemma:lowerbound}, 
        among all $r_i \ne r_j$ pairs, there exists an optimal $r_i^*, r_j^*$ pair where   $\langle r_i^*, r_j^* \rangle \ge \langle r_i, r_j \rangle$ $\forall r_i \ne r^*_i$ and $r_j \ne r^*_j$. We denote this maximum inner product as 
    \begin{equation}
        \ub = \langle r_i^*, r_j^* \rangle.
    \end{equation}
\end{enumerate}

\begin{proof} \hspace{1pt}
\begin{adjustwidth}{0.5cm}{0.0cm}
Given Fig.~\ref{app:fig:two_layer_img}, the equation for $\hsic_0$ is 
\begin{align}
    \hsic_0 
        & = \sums \Gij 
        e^{-\frac{(x_i - x_j)^TW W^T (x_i - x_j)}{2 \sigma_0^2}}
        -
        \sumsc |\Gij| 
        e^{-\frac{(x_i - x_j)^TW W^T (x_i - x_j)}{2 \sigma_0^2}}    \\
        & = 
        \sums \Gij
            \langle 
                \psi_{\sigma_0}(x_i), 
                \psi_{\sigma_0}(x_j)
            \rangle
            - \sumsc |\Gij|
            \langle 
                \psi_{\sigma_0}(x_i), 
                \psi_{\sigma_0}(x_j)
            \rangle
\end{align}
Notice that as $\sigma_0 \rightarrow 0$, we have 
\begin{equation}
    \lim_{\sigma_0 \rightarrow 0} 
    \langle 
        \psi_{\sigma_0}(x_i), 
        \psi_{\sigma_0}(x_j)
    \rangle   
    =  
    \begin{cases} 
    0 \quad \forall i \ne j
    \\
    1 \quad \forall i = j
    \end{cases}.
\end{equation}
In other words, as $\sigma_0 \rightarrow 0$, the samples $r_i$ in the RKHS of a Gaussian kernel approaches orthogonal to all other samples. Given this fact, it also implies that the $\sigma_0$ controls the inner product magnitude in RKHS space  of the maximum sample pair $r_i^*, r_j^*$. We define this maximum inner product as
\begin{equation}
    \langle 
        \psi_{\sigma_0}(x_i^*), 
        \psi_{\sigma_0}(x_j^*)
    \rangle  
    \ge 
     \langle 
        \psi_{\sigma_0}(x_i), 
        \psi_{\sigma_0}(x_j)
    \rangle      
\end{equation}
or equivalently
\begin{equation}
    \langle 
        r_i^*, 
        r_j^*
    \rangle  
    \ge 
     \langle 
        r_i,
        r_j
    \rangle      
\end{equation}

Therefore, given a $\sigma_0$, it controls the upper bound of the inner product. Notice that as $\sigma_0 \rightarrow 0$, every sample in RKHS becomes orthogonal. Therefore, the upper bound of $\langle r_i, r_j \rangle$ also approaches 0 when $r_i \ne r_j$. From this, we see the relationship
\begin{equation}
    \lim_{\sigma_0 \rightarrow 0} 
    \ub = \lim_{\sigma_0 \rightarrow 0}  \exp -(|.|/\sigma^{2}_{0}) = 0
\end{equation}, where $|.|$ is bounded and has a minimum and maximum, because we have finite number of samples.
\end{adjustwidth}
\end{proof}

\begin{lemma}
\label{app:lemma:as_ub_zero_L_approaches}
Given any fixed $\sigma_1 > 0 $, the lower bound $\lb\Lsigma$ is a function with respect to $\sigma_0$ and as $\sigma_0 \rightarrow 0$, $\lb\Lsigma$ approaches the function
    \begin{equation}
        \lb(\sigma_1) = 
        \sum_{g=1}^{\nclass} \sum_{\cS^g} \Gij
        - 
        \sum_{g_1 \ne g_2}^{\nclass}
        \sum_{i \in \cS^{g_1}} 
        \sum_{j \in \cS^{g_2}} 
        |\Gij|
        e^{-\frac{1}{\zeta \sigma_1^2}}.
    \end{equation}
At this point, if we let $\sigma_1 \rightarrow 0$, we have 
    \begin{align}
        \lim_{\sigma_1 \rightarrow 0}
        \lb(\sigma_1) & = \sums \Gij \\
        & = \hsic^*.
    \end{align}
\end{lemma}

\begin{proof} \hspace{1pt}
\begin{adjustwidth}{0.5cm}{0.0cm}
    Given \citelemma{app:lemma:ub_goes_to_zero}, we know that 
     \begin{equation}
        \lim_{\sigma_0 \rightarrow 0} 
        \ub = 0.
    \end{equation}   
    Therefore, having $\sigma_0 \rightarrow 0$ is equivalent to having $\ub \rightarrow 0$. Since 
    \citelemma{app:lemma:lowerbound} provide the equation of a lower bound that is a function of $\ub$,  this lemma is proven by simply evaluating $\lb\Lsigma$ as $\ub \rightarrow 0$. Following these steps, we have
    \begin{align}
        \lb(\sigma_1) & = 
        \lim_{\ub \rightarrow 0}
        \sum_{g=1}^{\nclass} \sum_{\cS^g} \Gij
        e^{-\frac{\mathscr{N}_g \ub^2}{\sigma_1^2}}
        - 
        \sum_{g_1 \ne g_2}^{\nclass} 
        \sum_{i \in \cS^{g_1}} 
        \sum_{j \in \cS^{g_2}} 
        |\Gij|
        e^{-\frac{\mathscr{N}_{g_1,g_2} (\ub)}{\sigma_1^2}}, \\
        & = 
        \sum_{g=1}^{\nclass} \sum_{\cS^g} \Gij
        - 
        \sum_{g_1 \ne g_2}^{\nclass} 
        \sum_{i \in \cS^{g_1}} 
        \sum_{j \in \cS^{g_2}} 
        |\Gij|
        e^{-\frac{1}{\zeta \sigma_1^2}}.
    \end{align}    
At this point, as $\sigma_1 \rightarrow 0$, our lower bound reaches the global maximum 
    \begin{align}
        \lim_{\sigma_1 \rightarrow 0}
        \lb(\sigma_1) & = \sum_{g=1}^{\nclass} \sum_{\cS^g} \Gij
        = \sums \Gij \\
        & = \hsic^*.
    \end{align}
    
\end{adjustwidth}
\end{proof}

\begin{lemma}
\label{app:lemma:arbitrarily_close}
Given any $\hsic_{l-2}$, $\delta > 0$, there exists a $\sigma_0 > 0 $ and $\sigma_1 > 0$ such that
\begin{equation}
        \hsic^{*} - \hsic_l \le \delta.
        \label{app:eq:proof_I}
\end{equation}
\end{lemma}
\begin{proof} $\hspace{1pt}$

\begin{adjustwidth}{0.5cm}{0.0cm}
\textbf{Observation 1. }

Note that the objective of $\hsic_l$ is
\begin{equation}
\begin{split}
    \hsic_l = 
    \max_W
    &
    \sums \Gij 
    e^{-\frac{\rij{\cS}{\cS}^T WW^T \rij{\cS}{\cS}}{2\sigma_1^2}} \\
    - & \sumsc |\Gij| 
     e^{-\frac{\rij{\cS^c}{\cS^c}^T WW^T \rij{\cS^c}{\cS^c}}{2\sigma_1^2}}.
\end{split}
\end{equation}
Since the Gaussian kernel is bounded between 0 and 1, the theoretical maximum of $\hsic^*$ is when the kernel is 1 for $\cS$ and 0 for $\cS^c$ with the theoretical maximum as $\hsic^* = \sums \Gij$. Therefore \eq{app:eq:proof_I} inequality is equivalent to 
\begin{equation}
        \sums \Gij - \hsic_l \le \delta.
\end{equation}
\end{adjustwidth}

\begin{adjustwidth}{0.5cm}{0.0cm}
\textbf{Observation 2. }

If we choose a $\sigma_0$ such that

\begin{equation}
    \lb^*(\sigma_1) - \lb\Lsigma \le \frac{\delta}{2}
    \quad \text{and} \quad
    \hsic^* - \lb^*(\sigma_1) \le \frac{\delta}{2}
\end{equation}
then we have identified the condition where $\sigma_0 > 0$ and $\sigma_1 > 0$ such that
\begin{equation}
    \sums \Gij - \lb\Lsigma \le \delta.
\end{equation}
Note that the $\lb^*(\sigma_1)$ is a continuous function of $\sigma_1$. Therefore, a $\sigma_1$ exists such that $\lb^*(\sigma_1)$ can be set arbitraty close to $\hsic^{*}$. Hence, we choose an $\sigma_{1}$ that has the following property:
\begin{equation}
    \hsic^* - \lb^*(\sigma_1) \le \frac{\delta}{2}.
\end{equation}
We next fix $\sigma_{1}$, we also know $\lb\Lsigma$ is a continuous function of $\sigma_{0}$, and it has a limit $\lb^*(\sigma_1)$ as $\sigma_{0}$ approaches to 0, hence there exits a $\sigma_{0}$, where 
\begin{equation}
    \lb^*(\sigma_1) - \lb\Lsigma \le \frac{\delta}{2}
\end{equation}
Then we have:
\begin{equation}
    \lb^*(\sigma_1) - \lb\Lsigma \le \frac{\delta}{2}
    \quad \text{and} \quad
    \hsic^* - \lb^*(\sigma_1) \le \frac{\delta}{2}.
\end{equation}
By adding the two $\frac{\delta}{2}$, we conclude the proof.
\end{adjustwidth}
\end{proof}

\begin{lemma}
    \label{app:lemma:approach_optimal_H1}
    There exists a \KS $\{\fm_{l^{\circ}}\}_{l=1}^L$ parameterized by a set of weights $W_l$ and a set of bandwidths $\sigma_l$ such that    
    \begin{equation}
        \lim_{l \rightarrow \infty}  \hsic_l = \hsic^* , \quad \hsic_{l+1} > \hsic_l \quad \forall l
    \end{equation}
\end{lemma}

Before, the proof, we use the following figure, Fig.~\ref{app:fig:all_layers}, to illustrate the relationship between \KS $\{\phi_{l^\circ}\}_{l=1}^L$ that generates the \RS $\{\hsic_l\}_{l=1}^L$. By solving a network greedily, we separate the network into $L$ separable problems. At each additional layer, we rely on the weights learned from the previous layer. At each network, we find $\sigma_{l-1}$, $\sigma_l$, and $W_l$ for the next network. We also note that since we only need to prove the existence of a solution, this proof is done by \textit{\textbf{Proof by Construction}}, i.e, we only need to show an example of its existence. Therefore, this proof consists of us constructing a \RS which satisfies the lemma.
        \begin{figure}[h]
        \center
            \includegraphics[width=7cm]{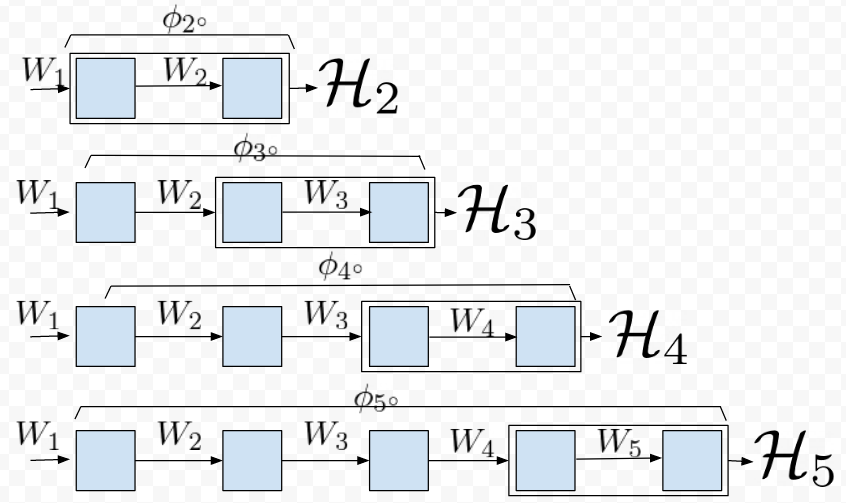}
            \caption{Relating \KS to \RS.}
            \label{app:fig:all_layers}
        \end{figure} 

\begin{proof} \hspace{1pt}
\begin{adjustwidth}{0.5cm}{0.0cm}
We first note that from \citelemma{app:lemma:arbitrarily_close}, we have previously proven given any $\hsic_{l-2}$, $\delta > 0$, there exists a $\sigma_0 > 0 $ and $\sigma_1 > 0$ such that
\begin{equation}
    \hsic^{*} - \hsic_{l} \leq \delta_{l}.
    \label{eq:cond1}
\end{equation}
This implies that based on Fig.~\ref{app:fig:all_layers}, at any given layer, we could reach arbitrarily close to $\hsic^*$. Given this, we list the 2 steps to build the \RS.


\textbf{Step 1: } Define $\{\mathcal{E}_{n}\}_{n =  1}^{\infty}$ as a sequence of numbers $\hsic^{*} - \frac{\hsic^{*} - \hsic_0}{n}$ on the real line. We have the following properties for this sequence:
\begin{equation}
  \lim_{n\rightarrow\infty} \mathcal{E}_{n} = \hsic^{*}
 , \quad \mathcal{E}_{1} = \mathcal{H}_{0}.   
\end{equation}

Using these two properties, for any $\hsic_{l-1} \in [\hsic_{0},\hsic^{*}]$ there exist an unique $n$, where  \begin{equation}
    \mathcal{E}_{n} \leq \hsic_{l-1} < \mathcal{E}_{n+1}.
    \label{eq:bounding_box}
\end{equation}

\textbf{Step 2: }  
    For any given $l$, we choose $\delta_{l}$ to satisfies \eq{eq:cond1} by the following procedure, First find an $n$ that satisfies
    \begin{equation}
        \mathcal{E}_{n} \leq \hsic_{l-1} <  \mathcal{E}_{n+1}, 
        \label{ineq:differential}
    \end{equation}
    and second define $\delta_l$ to be 
    \begin{equation}
        \delta_{l} = \hsic^{*} - \mathcal{E}_{n+1}.
        \label{eq:delta_define}
    \end{equation}

To satisfy \eq{eq:cond1}, the following must be true.  \begin{equation}
    \hsic^{*} - \hsic_{l-1} \leq \delta_{l-1}.
    \label{eq:delta_greater_than_l1}
\end{equation}
and further we found $n$ such that
\begin{equation}
    \mathcal{E}_{n} \leq \hsic_{l-1} <  \mathcal{E}_{n+1} \implies \hsic^{*} - \mathcal{E}_{n} \geq \hsic^{*} - \hsic_{l-1} > \hsic^{*} - \mathcal{E}_{n+1}.
    \label{eq:all_ineqals}
\end{equation}
Thus combining \eq{eq:delta_define}, \eq{eq:delta_greater_than_l1}, and \eq{eq:all_ineqals} we have
\begin{equation}
    \delta_{l-1} > \delta_{l}.
\end{equation}
Therefore, $\{ \delta_l \}$ is a decreasing sequence.

\textbf{Step 3: } 
Note that $\{\mathcal{E}_{n}\}$ is a converging sequence where 
\begin{equation}
    \lim_{n \rightarrow \infty}
    \hsic^{*} - \frac{\hsic^{*} - \hsic_0}{n} = \hsic^*.
\end{equation}
Therefore, $\{\Delta_n\} = \hsic^* - \{\mathcal{E}_n\}$ is also a converging sequence where 
\begin{equation}
    \lim_{n \rightarrow \infty}
    \hsic^* - \hsic^{*} + \frac{\hsic^{*} - \hsic_0}{n} = 0
\end{equation}
and $\{\delta_{l}\}$ is a subsequence of $\{\Delta_{l}\}$. Since any subsequence of a converging sequence also converges to the same limit, we know that
\begin{equation}
    \lim_{l \rightarrow \infty} \delta_l = 0.
\end{equation}


Following this construction, if we always choose $\hsic_l$ such that 
\begin{equation}
    \hsic^{*} - \hsic_{l} \leq \delta_{l}.
    \label{eq:key_inequality}
\end{equation}
As $l \rightarrow \infty$, the inequality becomes
\begin{align}
    \hsic^{*} - 
    \lim_{l \rightarrow \infty} 
    \hsic_{l} 
    &  \leq 
    \lim_{l \rightarrow \infty} 
    \delta_{l}, \\
    &  \leq 
    0.
\end{align}
Since we know that 
\begin{equation}
    \hsic^{*} - \hsic_l \geq 0\, \forall l. 
\end{equation}
The condition of 
\begin{equation}
    0 \leq
    \hsic^{*} - \lim_{l \rightarrow \infty} \hsic_l \leq 0
\end{equation}
is true only if 
\begin{equation}
    \hsic^{*} - \lim_{l \rightarrow \infty} \hsic_l = 0.
\end{equation}
This allows us to conclude 
\begin{equation}
    \hsic^{*} = \lim_{l \rightarrow \infty} \hsic_l.
\end{equation}




\textbf{Proof of the Monotonic Improvement. } 

Given \eq{eq:bounding_box} and \eq{eq:delta_define}, 
at each step we have the following: 
\begin{align}
    \hsic_{l-1} & < \mathcal{E}_{n+1} \\
    & \leq 
    \hsic^{*} - \delta_{l}. 
\end{align}
Rearranging this inequality, we have
\begin{equation}
    \delta_{l} < \hsic^{*} - \hsic_{l-1}.
    \label{eq:keyIneq}
\end{equation}
By combining the inequalities from \eq{eq:keyIneq} and  \eq{eq:key_inequality}, we have the following relationships.
\begin{align}
    \hsic^{*} - \hsic_{l} \leq \delta_{l} & < \hsic^{*} - \hsic_{l-1} \\
    \hsic^{*} - \hsic_{l} & < \hsic^{*} - \hsic_{l-1} \\
    - \hsic_{l} & < - \hsic_{l-1} \\
    \hsic_{l} & > \hsic_{l-1}, \\
\end{align}
which concludes the proof of theorem.

\end{adjustwidth}
\end{proof}

%
%
%

\end{appendices}

\begin{appendices} 
\section{Proof for Theorem \ref{thm:geometric_interpret}}
\label{app:thm:geometric_interpret}

\textbf{Theorem \ref{thm:geometric_interpret}: }
\textit{
As $l \rightarrow \infty$ and  $\hsic_l \rightarrow \hsic^*$, 
the following properties are satisfied: }

\begin{enumerate}[label=\Roman*]
    \item 
    the scatter ratio approaches 0 where
    \begin{equation}
    \lim_{l \rightarrow \infty} \frac{\Tr(S_w^{l})}{\Tr(S_b^{l})} =   0
\end{equation}

\item the \KS converges to the following kernel:
\begin{equation}
    \lim_{l \rightarrow \infty} 
    \kf(x_{i},x_{j})^{l} = 
    \kf^* =  
    \begin{cases} 
    0 \quad \forall i,j \in \mathcal{S}^c
    \\
    1 \quad \forall i,j \in \mathcal{S}
    \end{cases}.
\end{equation}
\end{enumerate}

\begin{proof}

We start by proving condition II starting from the $\hsic$ objective using a \rbfk
\begin{align}
    \max_{W} 
    \sums \Gij \kf_W(r_i, r_j) 
    -
    \sumsc |\Gij| \kf_W(r_i, r_j)\\
    \max_{W} 
    \sums \Gij  \ISMexp
    -
    \sumsc |\Gij| \ISMexp   
\end{align}
Given that $\hsic_l \rightarrow \hsic^*$,
and the fact that $0 \leq \kf_W \leq 1$,
this implies that the following condition must be true:
\begin{equation}
   \hsic^{*} = \sums \Gij =  
    \sums \Gij (1) 
    -
    \sumsc |\Gij| (0).
\end{equation}

Based on \eq{eq:cond1}, our construction at each layer ensures to satisfy 
\begin{equation}
    \hsic^{*} - \hsic_{l} \leq \delta_{l}.
\end{equation}
Substituting the definition of $\hsic^*$ and $\hsic_l$, we have
\begin{align}
    \sums \Gij (1) 
    -\left[\sums \Gij  \kf_W(r_i, r_j) 
    -
    \sumsc |\Gij| \kf_W(r_i, r_j) \right] \leq \delta_{l}
    \\
     \sums \Gij  (1-\kf_W(r_i, r_j)) 
    +
    \sumsc |\Gij| \kf_W(r_i, r_j)  \leq \delta_{l}.
    \label{eq:kernel_inequal}
\end{align}
Since every term within the summation in \eq{eq:kernel_inequal} is positive, this implies
\begin{align}\label{eq:limit kernal behaviour}
     1-\kf_W(r_i, r_j) \leq \delta_{l} \quad i,j \in \mathcal{S}
     \\
     \kf_W(r_i, r_j) \leq \delta_{l}\quad  i,j \in \mathcal{S}^{c}.
     \label{eq:limit kernal behaviour2}
\end{align}

So as $l \rightarrow \infty$ and  $\delta_{l} \rightarrow 0$, every component getting closer to limit Kernel, i.e, taking the limit from both sides and using the fact that is proven is theorem 1 $\lim_{l\rightarrow \infty} \delta_{l} = 0$ leads to
\begin{align}
   \lim_{l\rightarrow \infty}  1 \leq \kf_W(r_i, r_j)  \quad i,j \in \mathcal{S}
     \\
     \lim_{l\rightarrow \infty}\kf_W(r_i, r_j) \leq 0 \quad  i,j \in \mathcal{S}^{c}
\end{align}
both terms must instead be strictly equality. Therefore, we see that at the limit point $\kf_W$ would have the form 
\begin{equation}
     \kf^* =  
    \begin{cases} 
    0 \quad \forall i,j \in \mathcal{S}^c
    \\
    1 \quad \forall i,j \in \mathcal{S}
    \end{cases}.   
\end{equation}

\textbf{First Property}:

Using \eq{eq:limit kernal behaviour} and \eq{eq:limit kernal behaviour2} we have:

\begin{align}
         1-\delta_{l}\leq  \ISMexp  \quad i,j \in \mathcal{S}
     \\
     \ISMexp \leq \delta_{l}\quad  i,j \in \mathcal{S}^{c}.
\end{align}

As  $\lim_{l\rightarrow \infty} \delta_{l} = 0$, taking the limit from both side leads to:


\begin{equation}
    \begin{cases} 
    \ISMexp = 1 \quad \forall i,j \in \mathcal{S}\\
    \ISMexp = 0 \quad \forall i,j \in \mathcal{S}^c
    \end{cases}.   
\end{equation}
If we take the log of the conditions, we get
\begin{equation}
    \begin{cases} 
     \frac{1}{2\sigma^2}
    (r_i - r_j)^T W W^T(r_i - r_j) =0 
    \quad \forall i,j \in \mathcal{S}\\   
    \frac{1}{2\sigma^2}
    (r_i - r_j)^T W W^T(r_i - r_j) = \infty
    \quad \forall i,j \in \mathcal{S}^c
    \end{cases}.   
\end{equation}
This implies that as $l \rightarrow \infty$ we have
\begin{equation}
    \lim_{l \rightarrow \infty}
    \sums
    \frac{1}{2\sigma^2}
    (r_i - r_j)^T W W^T(r_i - r_j) 
    = 
    \lim_{l \rightarrow \infty}
    \Tr(S_w) = 0.
\end{equation}
\begin{equation}
    \lim_{l \rightarrow \infty}
    \sumsc \frac{1}{2\sigma^2}
    (r_i - r_j)^T W W^T(r_i - r_j) 
    = 
    \lim_{l \rightarrow \infty}
    \Tr(S_b)
    =
    \infty,
\end{equation}
This yields the ratio
\begin{equation}
    \lim_{\hsic_l \rightarrow \hsic^*} \frac{\Tr(S_w)}{\Tr(S_b)} = \frac{0}{\infty} = 0.
\end{equation}

\end{proof}
\end{appendices}
\begin{appendices}
\section{Proof for \texorpdfstring{$W_s$}\xspace Optimality}
\label{app:lemma:W_not_optimal}
\textit{Given $\hsic_l$ as the empirical risk at layer $l \ne L$, we have}
\begin{equation}
    \frac{\partial}{ \partial W_l}\hsic_l(W_s) \ne 0
\end{equation}

\begin{proof}
    Given $\frac{1}{\sqrt{\zeta}}$ as a normalizing constant for  $W_s = \frac{1}{\sqrt{\zeta}} \sum_{\alpha} r_{\alpha}$ such that $W^TW=I$. We start with the Lagrangian
    \begin{equation}
        \mathcal{L} = 
        - \sum_{i,j} \Gij \ISMexp - \Tr(\Lambda(W^TW - I)).
    \end{equation}
    If we now take the derivative with respect to the Lagrange, we get
    \begin{equation}
        \nabla \mathcal{L} = 
        \frac{1}{\sigma^2}
        \sum_{i,j} \Gij \ISMexp
        (r_i - r_j)(r_i - r_j)^TW 
        - 2W\Lambda.
    \end{equation}
    By setting the gradient to 0, we have
     \begin{align}
        \left[
        \frac{1}{2\sigma^2}
        \sum_{i,j} \Gij \ISMexp
        (r_i - r_j)(r_i - r_j)^T 
        \right]
        W 
        =&  W\Lambda. \\
        \mathcal{Q}_l W =& W \Lambda.
        \label{app:eq:ism_conclusion}
    \end{align}   
    From \eq{app:eq:ism_conclusion}, we see that $W$ is only the optimal solution when $W$ is the eigenvector of $Q_l$. Therefore, by setting $W$ to 
    $W_s = \frac{1}{\sqrt{\zeta}} \sum_{\alpha} r_{\alpha}$, it is not guaranteed to yield an optimal for all $\sigma_l$.
\end{proof}

\end{appendices}
\begin{appendices}
\section{Proof for Corollary \ref{corollary:mse} and \ref{corollary:ce}}
\label{app:corollary:ce}

\textbf{Corollary} \ref{corollary:mse}: 
    \textit{Given $\hsic_l \rightarrow \hsic^*$, the network output in IDS solves MSE via a translation of labels.}
   
\begin{proof} \hspace{1pt}
\begin{adjustwidth}{0.5cm}{0.0cm}
    As $\hsic_l \rightarrow \hsic^*$, Thm.~\ref{thm:geometric_interpret} shows that sample of the same class are mapped into the same point. Assuming that $\fm$ has mapped the sample into $c$ points $\alpha = [\alpha_1, ..., \alpha_c]$ that's different from the truth label
    $\xi = [\xi_1, ..., \xi_c]$. Then the $\mse$ objective is minimized by translating the $\fm$ output by 
    \begin{equation}
        \xi - \alpha.
    \end{equation}
\end{adjustwidth}
\end{proof}

\textbf{Corollary} \ref{corollary:ce}: 
    \textit{Given $\hsic_l \rightarrow \hsic^*$, the network output in RKHS solves $\ce$ via a change of bases.}
    
\textbf{Assumptions, and Notations. }
\begin{enumerate}
    \item  
        $n $ is the number of samples.
    \item
        $\nclass$ is the number of classes.
    \item
        $y_i \in \mathbb{R}^{\nclass}$ is the ground truth label for the $i^{th}$ sample. It is one-hot encoded where only the $j^{th}$ element is 1 if $x_i$ belongs to the $j^{th}$ class, all other elements would be 0.
    \item
        We denote $\fm$ as the network, and $\hat{y}_i \in \mathbb{R}^{\nclass}$ as the network output where $\hat{y}_i = \fm(x_i)$. We also assume that $\hat{y}_i$ is constrained on a probability simplex where $1 = \hat{y}_i^T \mathbf{1}_n$.
    \item
        We denote the $j^{th}$ element of $y_i$, and $\hat{y}_i$ as $y_{i,j}$ and $\hat{y}_{i,j}$ respectively.
    \item
        We define
    \begin{addmargin}[1em]{2em}
    \textbf{Orthogonality Condition: }
        A set of samples $\{\hat{y}_1, ..., \hat{y}_n\}$ satisfies the orthogonality condition if
        \begin{equation} 
        \begin{cases}
        \langle \hat{y_{i}}, \hat{y_{j}}\rangle =1 & \forall\quad i,j \textrm{ same class} \\
        \langle \hat{y_{i}}, \hat{y_{j}}\rangle=0 & \forall\quad i,j \textrm{ not in the same class}
        \end{cases}.
        \end{equation}
    \end{addmargin}
    \item
        We define the Cross-Entropy objective as 
        \begin{equation}
        \underset{\fm}{\argmin} -\sum_{i=1}^{n} \sum_{j=1}^{\nclass} y_{i,j} \log(\fm(x_{i})_{i,j}).
        \end{equation}
    \end{enumerate}
\begin{proof}\hspace{1pt}
\begin{adjustwidth}{0.5cm}{0.0cm}
From Thm.~\ref{thm:geometric_interpret}, we know that the network $\fm$ output, $\{ \hat{y}_1, \hat{y}_2, ..., \hat{y}_n \}$,  satisfy the orthogonality condition at $\hsic^*$. Then there exists a set of  orthogonal bases represented by $\Xi = [\xi_1, \xi_2, ..., \xi_c]$ that maps $\{ \hat{y}_1, \hat{y}_2, ..., \hat{y}_n \}$ to simulate the output of a softmax layer. Let $\xi_{i} = \hat{y}_{j} , j\in \cS^{i}$, i.e., for the $i_{th}$ class we arbitrary choose one of the samples from this class and assigns $\xi_i$ of that class to be equal to the sample's output. Realize in our problem we have $<\hat{y}_{i},\hat{y}_{i}> = 1$, so if $<\hat{y}_{i},\hat{y}_{j}> = 1$, then subtracting these two would lead to $<\hat{y}_{i},\hat{y}_{i}-\hat{y}_{j}> = 0$, which is the same as $\hat{y}_{i}=\hat{y}_{j}$.
So this representation is well-defined and its independent of choices of the sample from each group if they satisfy orthogonality condition.
Now we define transformed labels, $Y$ as:
\begin{equation}
    Y = \hat{Y} \Xi.
\end{equation}
Note that $Y = [y_1, y_2, ..., y_n]^T$ which each $y_{i}$ is a one hot vector representing the class membership of $i$ sample in $c$ classes.
Since given $\Xi$ as the change of basis, we can match $\hat{Y}$ to $Y$ exactly, $\ce$ is minimized.
\end{adjustwidth}
\end{proof}
\end{appendices}

\newpage
\begin{appendices}
\section{Dataset Details}
\label{app:data_detail}

No samples were excludes from any of the dataset. 

\textbf{Wine. } 
    This dataset has 13 features, 178 samples, and 3 classes. The features are continuous and heavily unbalanced in magnitude. The dataset can be downloaded at \url{https://archive.ics.uci.edu/ml/datasets/wine.}

 \textbf{Divorce. } 
    This dataset has 54 features, 170 samples, and 2 classes. The features are discrete and balanced in magnitude.  The dataset can be downloaded at \url{https://archive.ics.uci.edu/ml/datasets/Divorce+Predictors+data+set.}
    
 \textbf{Car. } 
     This dataset has 6 features, 1728 samples and 2 classes. The features are discrete and balanced in magnitude.  The dataset can be downloaded at \url{https://archive.ics.uci.edu/ml/datasets/Car+Evaluation.}       
 
\textbf{Cancer. } 
    This dataset has 9 features, 683 samples, and 2 classes. The features are discrete and unbalanced in magnitude. The dataset can be downloaded at \url{https://archive.ics.uci.edu/ml/datasets/Breast+Cancer+Wisconsin+(Diagnostic)}.
    
 \textbf{Face. } 
    This dataset consists of images of 20 people in various poses. The 624 images are vectorized into 960 features.  
    The dataset can be downloaded at 
    \url{https://archive.ics.uci.edu/ml/datasets/CMU+Face+Images}.

\textbf{Random. } 
    This dataset has 2 features, 80 samples and 2 classes. It is generate with a gaussian distribution where half of the samples are randomly labeled as 1 or 0.
    
\textbf{Adversarial. } 
    This dataset has 2 features, 80 samples and 2 classes. It is generate with the following code:
    \begin{lstlisting}    
    #!/usr/bin/env python
    
    n = 40
    X1 = np.random.rand(n,2)
    X2 = X1 + 0.01*np.random.randn(n,2)
    
    X = np.vstack((X1,X2))
    Y = np.vstack(( np.zeros((n,1)), np.ones((n,1)) ))
    \end{lstlisting}    
 
\textbf{CFAR10 Test. } The test set images from CIFAR10 are preprocessed with a convolutional layer that outputs vectorized samples of $x_i \in \mathbb{R}^{10}$. This dataset has 10 features and 10,000 samples.  The preprocessing code to map the images to $\mathbb{R}^{10}$ data is included in the supplementary. The link to download the data is at \url{https://www.cs.toronto.edu/~kriz/cifar.html}.

\textbf{Raman. } The dataset consists of 4306 samples, 700 frequencies, and 35 different cell types. Since this is proprietary data, a download link is not included. 

\end{appendices}
\newpage
\begin{appendices}
\section{Optimal Gaussian \texorpdfstring{$\sigma $ }\xspace for Maximum Kernel Separation}
\label{app:opt_sigma}
Although the Gaussian kernel is the most common kernel choice for kernel methods, its $\sigma$ value is a hyperparameter that must be tuned for each dataset. This work proposes to set the $\sigma$ value based on the maximum kernel separation. The source code is made publicly available on \url{https://github.com/anonamous}.

Let $X \in \mathbb{R}^{n \times d}$ be a dataset of $n$ samples with $d$ features and let $Y \in \mathbb{R}^{n \times \nclass}$ be the corresponding one-hot encoded labels where $\nclass$ denotes the number of classes. Given $\kappa_X(\cdot, \cdot)$ and $\kappa_Y(\cdot,\cdot)$ as two kernel functions that applies respectively to $X$ and $Y$ to construct kernel matrices $K_X \in \mathbb{R}^{n \times n}$ and $K_Y \in \mathbb{R}^{n \times n}$. Given a set $\mathcal{S}$, we denote $|\mathcal{S}|$ as the number of elements within the set. Also let $\mathcal{S}$ and $\mathcal{S}^c$ be sets of all pairs of samples of $(x_i,x_j)$ from a dataset $X$ that belongs to the same and different classes respectively, then the average kernel value for all $(x_i,x_j)$ pairs with the same class is
\begin{equation}
    d_{\mathcal{S}} = \frac{1}{|\mathcal{S}|}\sum_{i,j \in \mathcal{S}} e^{-\frac{||x_i - x_j||^2}{2\sigma^2}}
\end{equation}
and the average kernel value for all $(x_i,x_j)$ pairs between different classes is
\begin{equation}
    d_{\mathcal{S}^c} = 
    \frac{1}{|\mathcal{S}^c|}\sum_{i,j \in \mathcal{S}^c} e^{-\frac{||x_i - x_j||^2}{2\sigma^2}}. 
\end{equation}
We propose to find the $\sigma$ that maximizes the difference between $d_{\mathcal{S}}$ and $d_{\mathcal{S}^c}$ or 
\begin{equation}
    \underset{\sigma}{\max} \quad 
    \frac{1}{|\mathcal{S}|}\sum_{i,j \in \mathcal{S}} e^{-\frac{||x_i - x_j||^2}{2\sigma^2}} - 
    \frac{1}{|\mathcal{S}^c|}\sum_{i,j \in \mathcal{S}^c} e^{-\frac{||x_i - x_j||^2}{2\sigma^2}}.
    \label{eq:main_objective}
\end{equation}
It turns out that is expression can be computed efficiently. Let $g = \frac{1}{|\mathcal{S}|}$ and $\bar{g} = \frac{1}{|\mathcal{S}^c|}$, and let $\textbf{1}_{n \times n} \in \mathbb{R}^{n \times n}$ be a matrix of 1s, then we can define $Q$ as
\begin{equation}
    Q = -g K_Y + \bar{g} (\textbf{1}_{n \times n} - K_Y).
\end{equation}
Or $Q$ can be written more compactly as
\begin{equation}
    Q = \bar{g} \textbf{1}_{n \times n} - (g + \bar{g})K_Y. 
\end{equation}
Given $Q$, Eq.~(\ref{eq:main_objective}) becomes
\begin{equation}
    \underset{\sigma}{\min} \quad 
    \Tr(K_X Q).
    \label{eq:obj_compact}
\end{equation}
This objective can be efficiently solved with BFGS. 

Below in Fig.~\ref{fig:max_kernel_separation}, we plot out the average within cluster kernel and the between cluster kernel values as we vary $\sigma$. From the plot, we can see that the maximum separation is discovered via BFGS. 
    \begin{figure}[h]
        \centering
        \includegraphics[width=10cm,height=7cm]{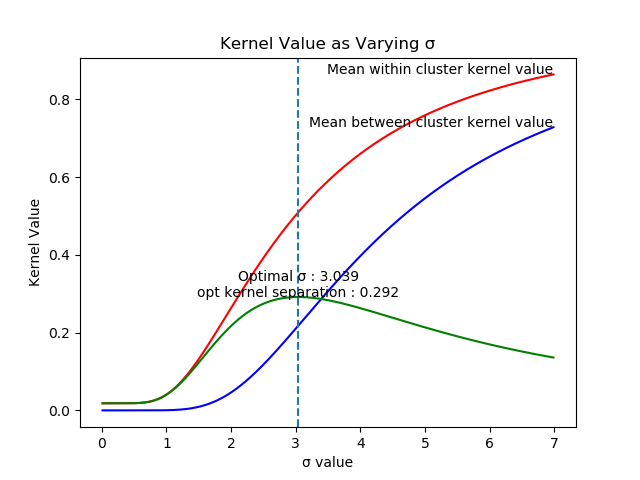}
        \caption{Maximum Kernel separation.}
        \label{fig:max_kernel_separation}
    \end{figure}

\textbf{Relation to HSIC. }    
From Eq.~(\ref{eq:obj_compact}), we can see that the $\sigma$ that causes maximum kernel separation is directly related to HSIC. Given that the HSIC objective is normally written as
\begin{equation}
    \underset{\sigma}{\min} \quad 
    \Tr(K_X H K_Y H),
\end{equation}
by setting $Q=HK_YH$, we can see how the two formulations are related. While the maximum kernel separation places the weight of each sample pair equally, HSIC weights the pair differently.
We also notice that the $Q_{i,j}$ element is positive/negative for $(x_i,x_j)$ pairs that are with/between classes respectively. Therefore, the argument for the global optimum should be relatively close for both objectives. Below in Figure~\ref{fig:max_HSIC}, we show a figure of HSIC values as we vary $\sigma$. Notice how the optimal $\sigma$ is almost equivalent to the solution from maximum kernel separation. For the purpose of \RS, we use $\sigma$ that maximizes the HSIC value.
    \begin{figure}[h]
        \centering
        \includegraphics[width=10cm,height=7cm]{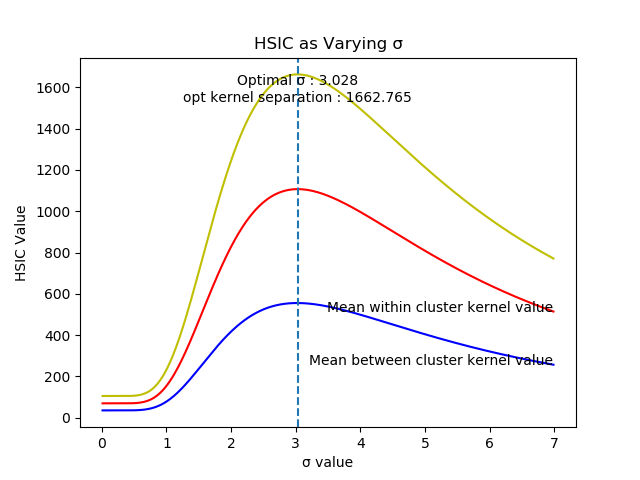}
        \caption{Maximal HSIC.}
        \label{fig:max_HSIC}
    \end{figure} 
\end{appendices}
\clearpage
\newpage
\begin{appendices}
\section{\texorpdfstring{$W_l$ }\xspace  Dimensions for each 10 Fold of each Dataset}
\label{app:W_dimensions}
We report the input and output dimensions of each $W_l$ for every layer of each dataset in the form of $(\alpha, \beta)$; the corresponding dimension becomes $W_l \in \mathbb{R}^{\alpha \times \beta}$. Since each dataset consists of 10-folds, the network structure for each fold is reported. We note that the input of the 1st layer is the dimension of the original data. However, after the first layer, the width of the RFF becomes the output of each layer; here we use 300. 

The $\beta$ value is chosen during the ISM algorithm. By keeping only the most dominant eigenvector of the $\Phi$ matrix, the output dimension of each layer corresponds with the rank of $\Phi$. It can be seen from each dataset that the first layer significantly expands the rank. The expansion is generally followed by a compression of fewer and fewer eigenvalues. These results conform with the observations made by \citet{montavon2011kernel} and \citet{ansuini2019intrinsic}.

\begin{tabular}{ll}
\centering
\tiny
\setlength{\tabcolsep}{7.0pt}
\renewcommand{\arraystretch}{1.2}
\begin{tabular}{cccccc|}
	\hline
Data & Layer 1 & Layer 2 & Layer 3 & Layer 4 \\ 
	\hline
adversarial 1 & (2, 2) & (300, 61) & (300, 35) \\ 
adversarial 2 & (2, 2) & (300, 61) & (300, 35) \\ 
adversarial 3 & (2, 2) & (300, 61) & (300, 8) & (300, 4) \\ 
adversarial 4 & (2, 2) & (300, 61) & (300, 29) \\ 
adversarial 5 & (2, 2) & (300, 61) & (300, 29) \\ 
adversarial 6 & (2, 2) & (300, 61) & (300, 7) & (300, 4) \\ 
adversarial 7 & (2, 2) & (300, 61) & (300, 34) \\ 
adversarial 8 & (2, 2) & (300, 12) & (300, 61) & (300, 30) \\ 
adversarial 9 & (2, 2) & (300, 61) & (300, 33) \\ 
adversarial 10 & (2, 2) & (300, 61) & (300, 33) \\ 
	\hline
\end{tabular}
&
\centering
\tiny
\setlength{\tabcolsep}{3.0pt}
\renewcommand{\arraystretch}{1.2}
\begin{tabular}{ccccc|}
	\hline
Data & Layer 1 & Layer 2 & Layer 3 \\ 
	\hline
Random 1 & (3, 3) & (300, 47) & (300, 25) \\ 
Random 2 & (3, 3) & (300, 46) & (300, 25) \\ 
Random 3 & (3, 3) & (300, 46) & (300, 25) \\ 
Random 4 & (3, 3) & (300, 47) & (300, 4) \\ 
Random 5 & (3, 3) & (300, 47) & (300, 25) \\ 
Random 6 & (3, 3) & (300, 45) & (300, 23) \\ 
Random 7 & (3, 3) & (300, 45) & (300, 25) \\ 
Random 8 & (3, 3) & (300, 45) & (300, 21) \\ 
Random 9 & (3, 3) & (300, 45) & (300, 26) \\ 
Random 10 & (3, 3) & (300, 47) & (300, 25) \\ 
	\hline
\end{tabular}
\end{tabular}

\begin{tabular}{ll}
\tiny
\setlength{\tabcolsep}{3.0pt}
\renewcommand{\arraystretch}{1.2}
\begin{tabular}{cccccccc|}
	\hline
Data & Layer 1 & Layer 2 & Layer 3 & Layer 4 & Layer 5 & Layer 6 \\ 
	\hline
spiral 1 & (2, 2) & (300, 15) & (300, 6) & (300, 7) & (300, 6) \\ 
spiral 2 & (2, 2) & (300, 13) & (300, 6) & (300, 7) & (300, 6) & (300, 6) \\ 
spiral 3 & (2, 2) & (300, 12) & (300, 6) & (300, 7) & (300, 6) & (300, 6) \\ 
spiral 4 & (2, 2) & (300, 13) & (300, 6) & (300, 7) & (300, 6) & (300, 6) \\ 
spiral 5 & (2, 2) & (300, 13) & (300, 6) & (300, 7) & (300, 6) \\ 
spiral 6 & (2, 2) & (300, 14) & (300, 6) & (300, 7) & (300, 6) \\ 
spiral 7 & (2, 2) & (300, 14) & (300, 6) & (300, 7) & (300, 6) \\ 
spiral 8 & (2, 2) & (300, 14) & (300, 6) & (300, 7) & (300, 6) & (300, 6) \\ 
spiral 9 & (2, 2) & (300, 13) & (300, 6) & (300, 7) & (300, 6) \\ 
spiral 10 & (2, 2) & (300, 14) & (300, 6) & (300, 7) & (300, 6) \\ 
	\hline
\end{tabular}
&
\tiny
\setlength{\tabcolsep}{3.0pt}
\renewcommand{\arraystretch}{1.2}
\begin{tabular}{cccccccc}
	\hline
Data & Layer 1 & Layer 2 & Layer 3 & Layer 4 & Layer 5 & Layer 6 \\ 
	\hline
wine 1 & (13, 11) & (300, 76) & (300, 6) & (300, 7) & (300, 6) & (300, 6) \\ 
wine 2 & (13, 11) & (300, 76) & (300, 6) & (300, 6) & (300, 6) & (300, 6) \\ 
wine 3 & (13, 11) & (300, 75) & (300, 6) & (300, 7) & (300, 6) & (300, 6) \\ 
wine 4 & (13, 11) & (300, 76) & (300, 6) & (300, 6) & (300, 6) & (300, 6) \\ 
wine 5 & (13, 11) & (300, 74) & (300, 6) & (300, 7) & (300, 6) & (300, 6) \\ 
wine 6 & (13, 11) & (300, 74) & (300, 6) & (300, 6) & (300, 6) & (300, 6) \\ 
wine 7 & (13, 11) & (300, 74) & (300, 6) & (300, 6) & (300, 6) & (300, 6) \\ 
wine 8 & (13, 11) & (300, 75) & (300, 6) & (300, 7) & (300, 6) & (300, 6) \\ 
wine 9 & (13, 11) & (300, 75) & (300, 6) & (300, 8) & (300, 6) & (300, 6) \\ 
wine 10 & (13, 11) & (300, 76) & (300, 6) & (300, 7) & (300, 6) & (300, 6) \\ 
	\hline
\end{tabular}
\end{tabular}

\begin{tabular}{ll}
\tiny
\setlength{\tabcolsep}{3.0pt}
\renewcommand{\arraystretch}{1.2}
\begin{tabular}{cccccccc|}
	\hline
Data & Layer 1 & Layer 2 & Layer 3 & Layer 4 & Layer 5 & Layer 6 \\ 
	\hline
car 1 & (6, 6) & (300, 96) & (300, 6) & (300, 8) & (300, 6) \\ 
car 2 & (6, 6) & (300, 96) & (300, 6) & (300, 8) & (300, 6) \\ 
car 3 & (6, 6) & (300, 91) & (300, 6) & (300, 8) & (300, 6) \\ 
car 4 & (6, 6) & (300, 88) & (300, 6) & (300, 8) & (300, 6) & (300, 6) \\ 
car 5 & (6, 6) & (300, 94) & (300, 6) & (300, 8) & (300, 6) \\ 
car 6 & (6, 6) & (300, 93) & (300, 6) & (300, 7) \\ 
car 7 & (6, 6) & (300, 92) & (300, 6) & (300, 8) & (300, 6) \\ 
car 8 & (6, 6) & (300, 95) & (300, 6) & (300, 7) & (300, 6) \\ 
car 9 & (6, 6) & (300, 96) & (300, 6) & (300, 9) & (300, 6) \\ 
car 10 & (6, 6) & (300, 99) & (300, 6) & (300, 8) & (300, 6) \\ 
	\hline
\end{tabular}
&
\tiny
\setlength{\tabcolsep}{3.0pt}
\renewcommand{\arraystretch}{1.2}
\begin{tabular}{ccccccc|}
	\hline
Data & Layer 1 & Layer 2 & Layer 3 & Layer 4 & Layer 5 \\ 
	\hline
divorce 1 & (54, 35) & (300, 44) & (300, 5) & (300, 5) \\ 
divorce 2 & (54, 35) & (300, 45) & (300, 4) & (300, 4) \\ 
divorce 3 & (54, 36) & (300, 49) & (300, 6) & (300, 6) \\ 
divorce 4 & (54, 36) & (300, 47) & (300, 7) & (300, 6) \\ 
divorce 5 & (54, 35) & (300, 45) & (300, 6) & (300, 6) \\ 
divorce 6 & (54, 36) & (300, 47) & (300, 6) & (300, 6) \\ 
divorce 7 & (54, 35) & (300, 45) & (300, 6) & (300, 6) & (300, 4) \\ 
divorce 8 & (54, 36) & (300, 47) & (300, 6) & (300, 7) & (300, 4) \\ 
divorce 9 & (54, 36) & (300, 47) & (300, 5) & (300, 5) \\ 
divorce 10 & (54, 36) & (300, 47) & (300, 6) & (300, 6) \\ 
	\hline
\end{tabular}
\end{tabular}

\begin{table}[h]
\tiny
\setlength{\tabcolsep}{3.0pt}
\renewcommand{\arraystretch}{1.2}
\begin{tabular}{cccccccccccc|}
	\hline
Data & Layer 1 & Layer 2 & Layer 3 & Layer 4 & Layer 5 & Layer 6 & Layer 7 & Layer 8 & Layer 9 & Layer 10 \\ 
	\hline
cancer 1 & (9, 8) & (300, 90) & (300, 5) & (300, 6) & (300, 6) & (300, 5) & (300, 4) & (300, 5) & (300, 6) & (300, 6) \\ 
cancer 2 & (9, 8) & (300, 90) & (300, 6) & (300, 7) & (300, 8) & (300, 11) & (300, 8) & (300, 4) \\ 
cancer 3 & (9, 8) & (300, 88) & (300, 5) & (300, 6) & (300, 7) & (300, 7) & (300, 6) & (300, 4) \\ 
cancer 4 & (9, 8) & (300, 93) & (300, 6) & (300, 7) & (300, 9) & (300, 11) & (300, 8) \\ 
cancer 5 & (9, 8) & (300, 93) & (300, 9) & (300, 10) & (300, 10) & (300, 11) & (300, 9) & (300, 7) \\ 
cancer 6 & (9, 8) & (300, 92) & (300, 7) & (300, 8) & (300, 8) & (300, 7) & (300, 7) \\ 
cancer 7 & (9, 8) & (300, 90) & (300, 4) & (300, 4) & (300, 5) & (300, 6) & (300, 6) & (300, 6) & (300, 6) \\ 
cancer 8 & (9, 8) & (300, 88) & (300, 5) & (300, 6) & (300, 7) & (300, 8) & (300, 7) & (300, 6) \\ 
cancer 9 & (9, 8) & (300, 88) & (300, 5) & (300, 7) & (300, 7) & (300, 7) & (300, 7) \\ 
cancer 10 & (9, 8) & (300, 97) & (300, 9) & (300, 11) & (300, 12) & (300, 13) & (300, 6) \\ 
	\hline
\end{tabular}
\end{table}

\begin{table}[h]
\tiny
\setlength{\tabcolsep}{3.0pt}
\renewcommand{\arraystretch}{1.2}
\begin{tabular}{cccccc|}
	\hline
Data & Layer 1 & Layer 2 & Layer 3 & Layer 4 \\ 
	\hline
face 1 & (960, 233) & (300, 74) & (300, 73) & (300, 46) \\ 
face 2 & (960, 231) & (300, 75) & (300, 73) & (300, 43) \\ 
face 3 & (960, 231) & (300, 76) & (300, 73) & (300, 44) \\ 
face 4 & (960, 232) & (300, 76) & (300, 74) & (300, 44) \\ 
face 5 & (960, 231) & (300, 77) & (300, 73) & (300, 43) \\ 
face 6 & (960, 232) & (300, 74) & (300, 72) & (300, 47) \\ 
face 7 & (960, 232) & (300, 76) & (300, 73) & (300, 45) \\ 
face 8 & (960, 230) & (300, 74) & (300, 74) & (300, 44) \\ 
face 9 & (960, 233) & (300, 76) & (300, 76) & (300, 45) \\ 
face 10 & (960, 231) & (300, 76) & (300, 70) & (300, 43) \\ 
	\hline
\end{tabular}
\end{table}

\end{appendices}

\clearpage
\newpage

\begin{appendices}
\section{Sigma Values used for Random and Adversarial Simulation}
\label{app:sigma_values}

The simulation of Thm.~\ref{thm:hsequence} as shown in Fig.~\ref{fig:thm_proof} spread the improvement across multiple layers. The $\sigma_l$ and $\hsic_l$ values are recorded here. We note that $\sigma_l$ are reasonably large and not approaching 0 and the improvement of $\hsic_l$ is monotonic. 

\begin{figure}[h]
\center
    \includegraphics[width=9cm]{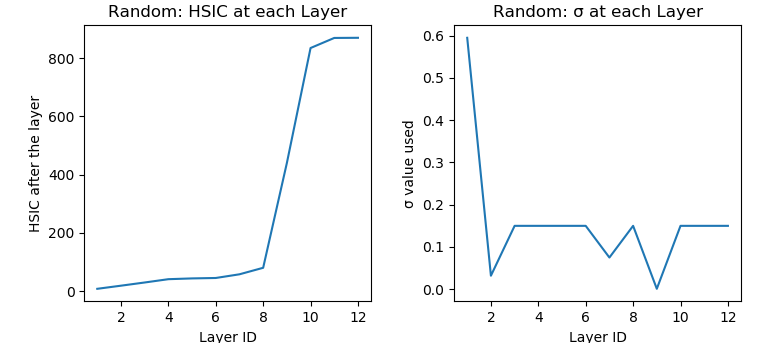}
    \caption{}
\end{figure} 

\begin{figure}[h]
\center
    \includegraphics[width=9cm]{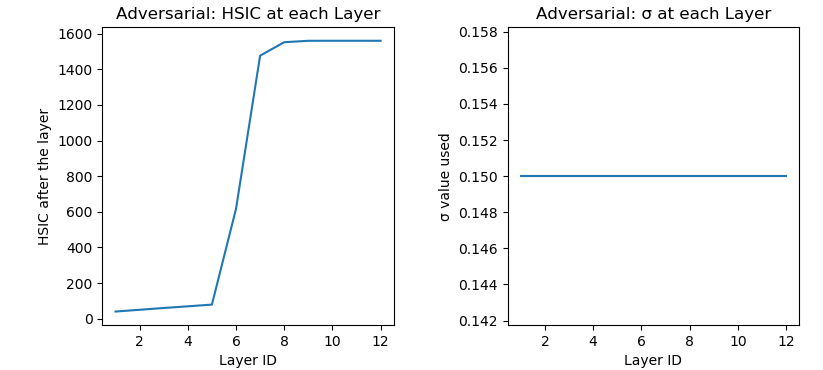}
    \caption{}
\end{figure}


Given a sufficiently small $\sigma_0$ and $\sigma_1$, Thm.~\ref{thm:hsequence} claims that it can come arbitrarily close to the global optimal using a minimum of 2 layers. We here simulate 2 layers using a relatively small $\sigma$ values ($\sigma_0 = 10^{-5}$) on the Random (left) and Adversarial (right) data and display the results of the 2 layers below. Notice that given 2 layer, it generated a clearly separable clusters that are pushed far apart.

    \begin{figure}[!h]
    \begin{minipage}[H]{6.0cm}
        \includegraphics[width=6cm]{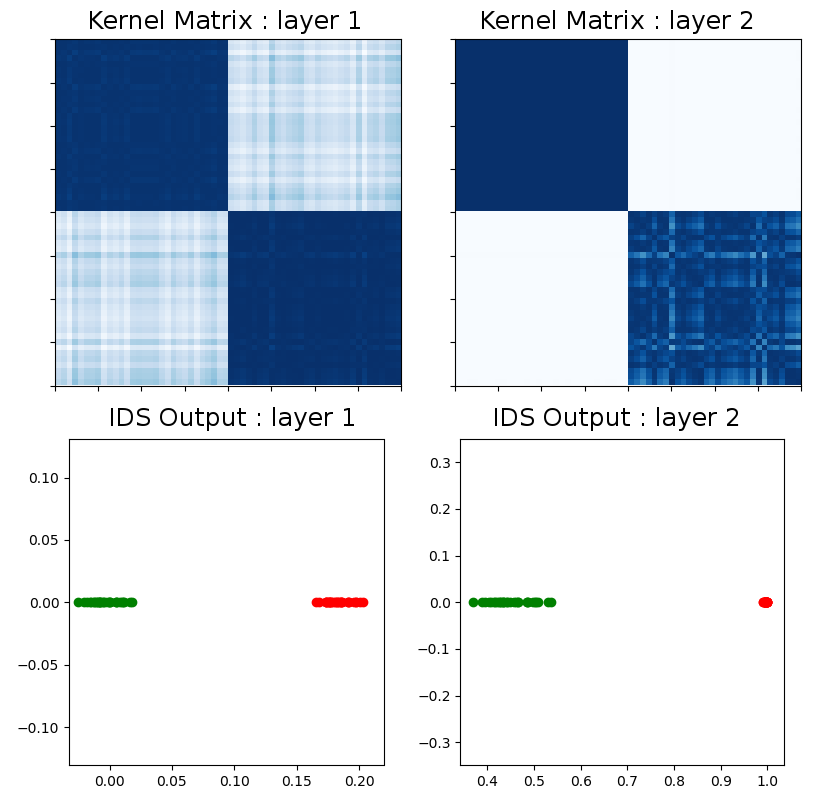}
        \caption{Random Dataset with 2\\layers and $\sigma=10^{-5}$}
    \end{minipage}%
    \begin{minipage}[H]{6.0cm}
        \includegraphics[width=6cm]{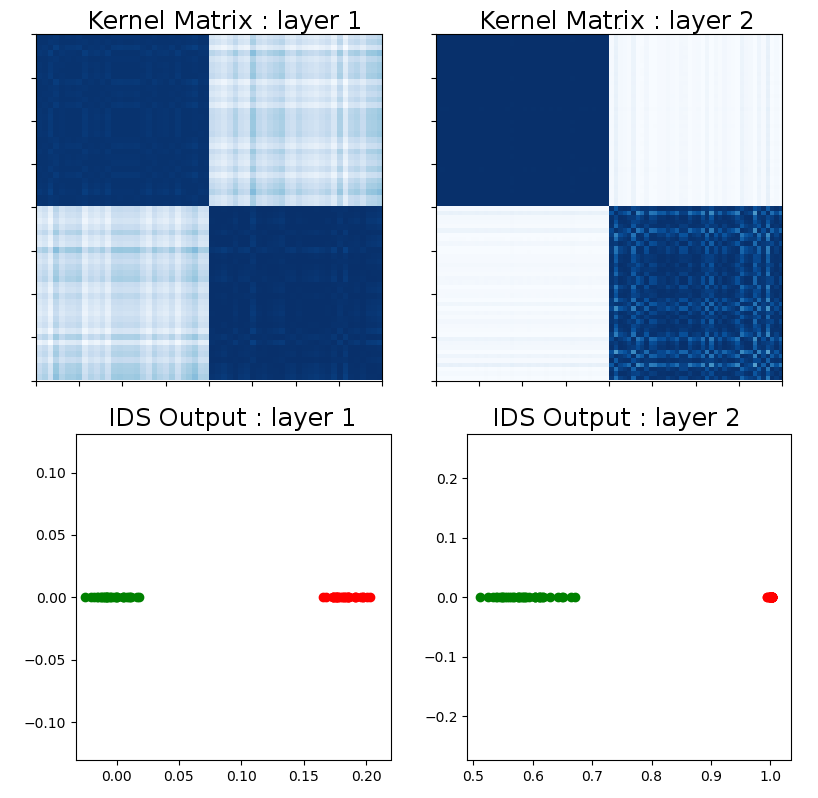}
        \caption{Adversarial Dataset with 2 \\layers and $\sigma=10^{-5}$}
    \end{minipage}
    \end{figure}          

\end{appendices}

\clearpage
\newpage

\begin{appendices}
\section{Evaluation Metrics Graphs}
\label{app:metric_graphs}

\begin{figure}[h]
\center
    \includegraphics[width=10cm]{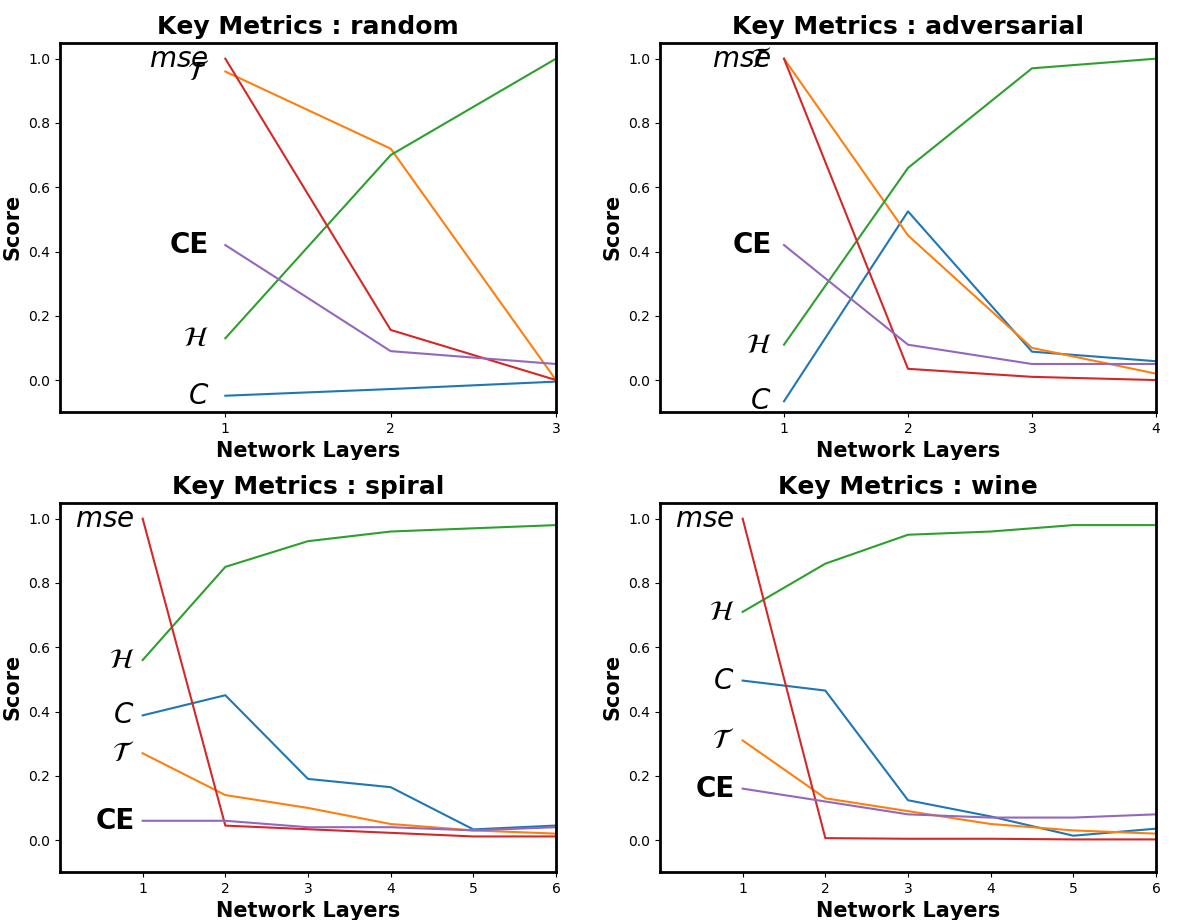}
    \caption{}
\end{figure} 

\begin{figure}[h]
\center
    \includegraphics[width=10cm]{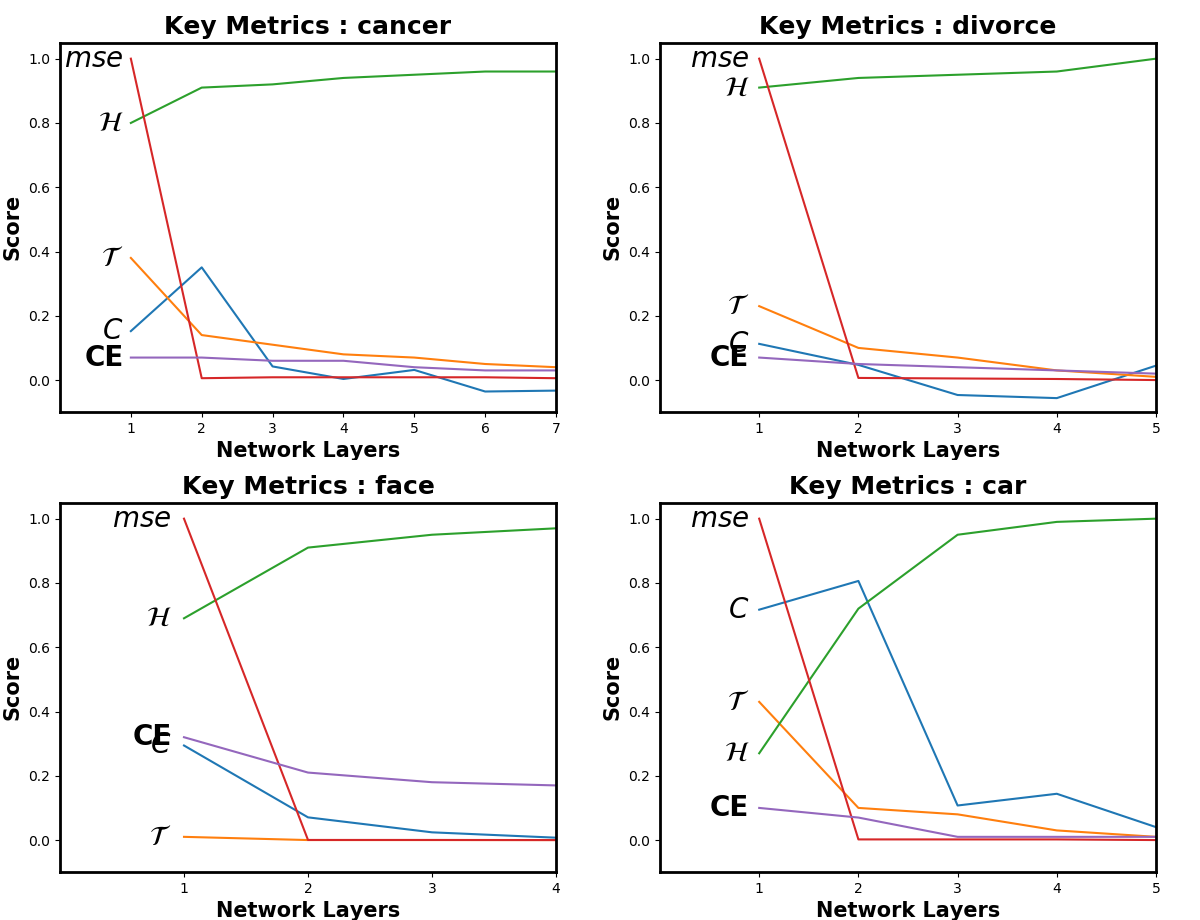}
    \caption{Figures of key metrics for all datasets as samples progress through the network. It is important to notice the uniformly and monotonically increasing \RS for each plot
    since this guarantees a converging kernel/risk sequence. As the $\mathcal{T}$ approach 0, samples of the same/difference classes in IDS are being pulled into a single point or pushed maximally apart respectively. As $C$ approach 0, samples of the same/difference classes in RKHS are being pulled into 0 or $\frac{\pi}{2}$ cosine similarity respectively.}
\end{figure} 

\end{appendices}

\clearpage

\begin{appendices}
\section{Graphs of Kernel Sequences}
\label{app:kernel_sequence_graph}

A representation of the \KS are displayed in the figures below for each dataset. The samples of the kernel matrix are previously organized to form a block structure by placing samples of the same class adjacent to each other. Since the Gaussian kernel is restricted to values between 0 and 1, we let white and dark blue be 0 and 1 respectively where the gradients reflect values in between. Our theorems predict that the \KS will evolve from an uninformative kernel into a highly discriminating kernel of perfect block structures. 

\begin{figure}[h]
\center
    \includegraphics[width=9cm]{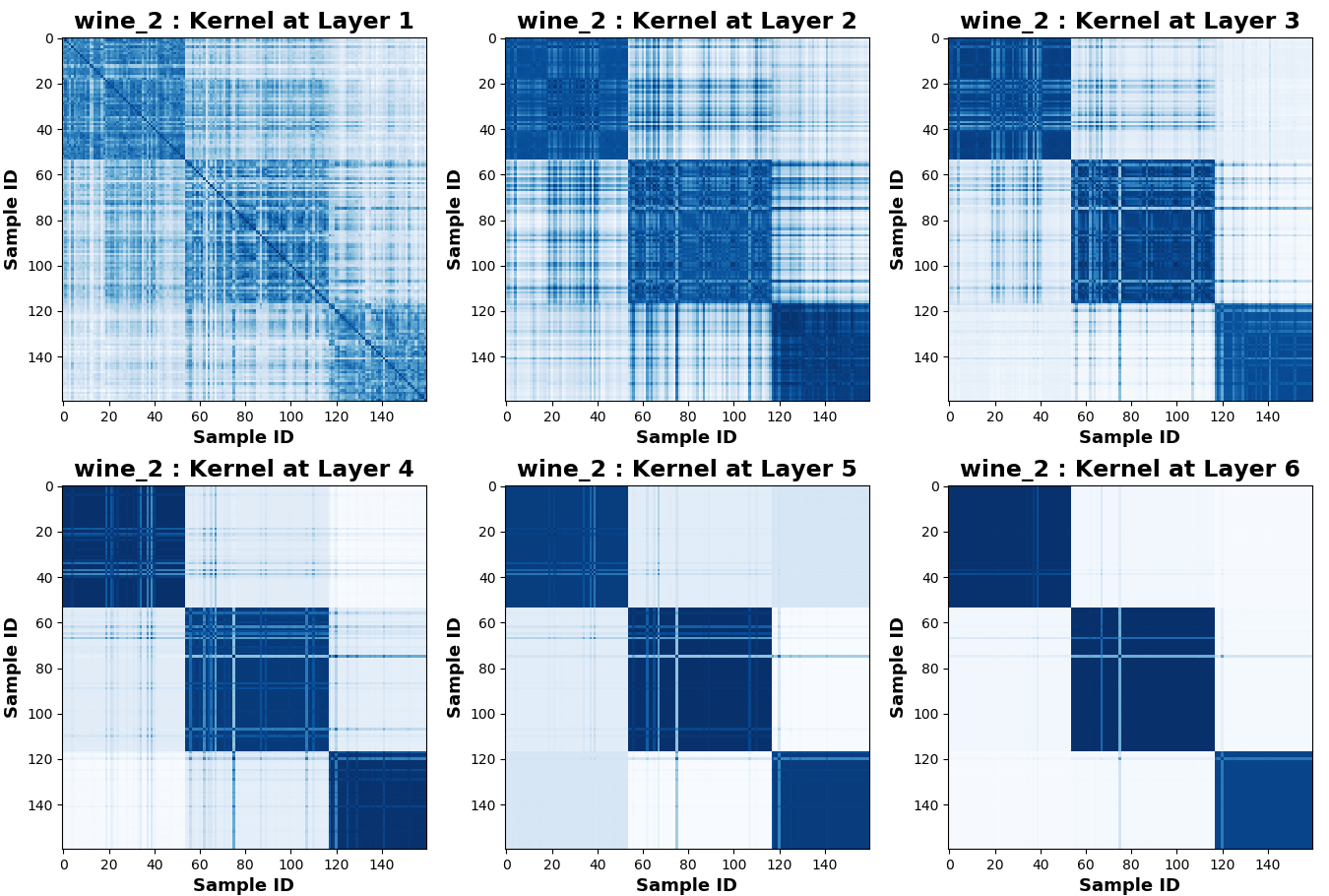}
    \caption{The kernel sequence for the wine dataset.}
\end{figure} 

\begin{figure}[h]
\center
    \includegraphics[width=9cm]{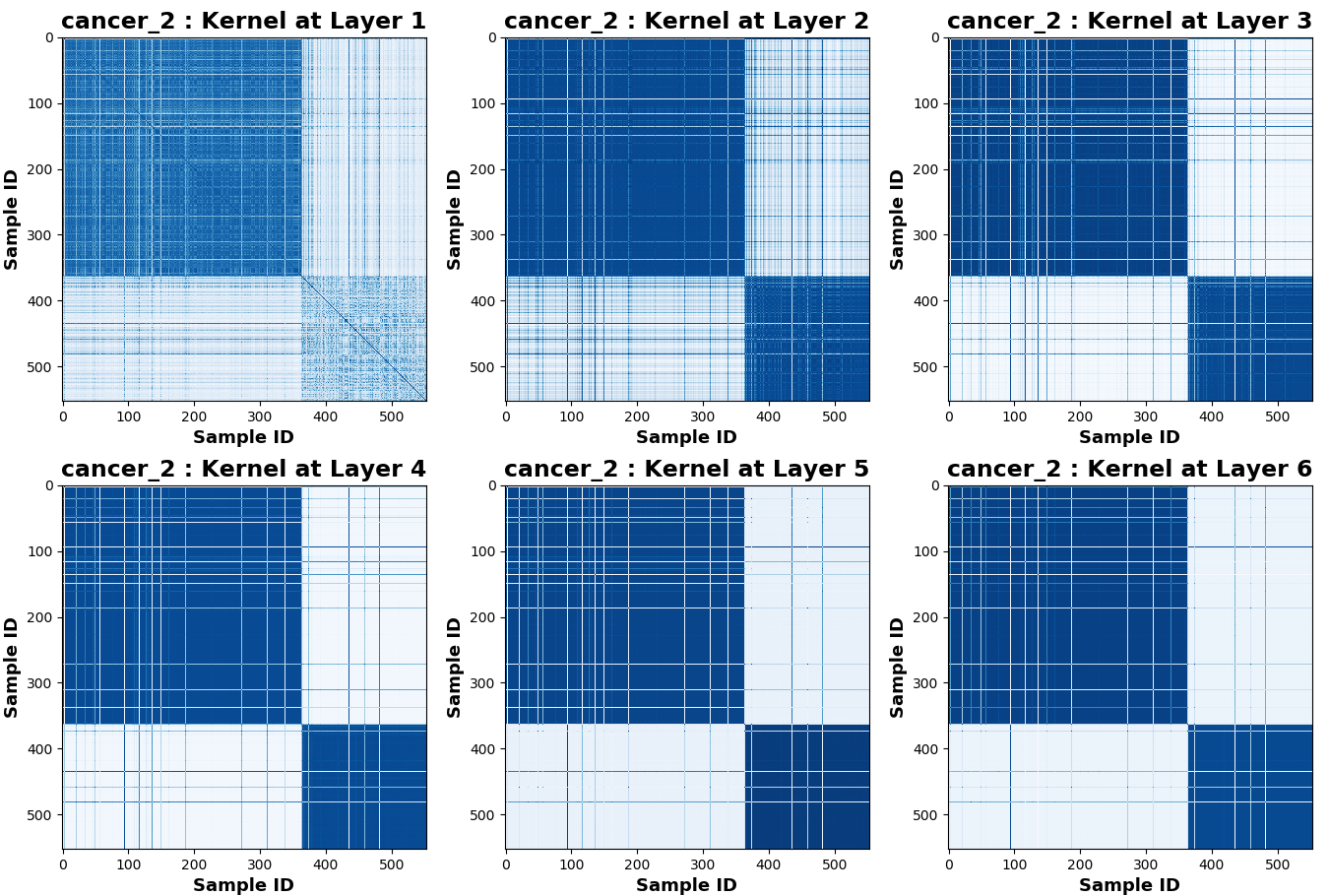}
    \caption{The kernel sequence for the cancer dataset.}
\end{figure} 

\begin{figure}[h]
\center
    \includegraphics[width=9cm]{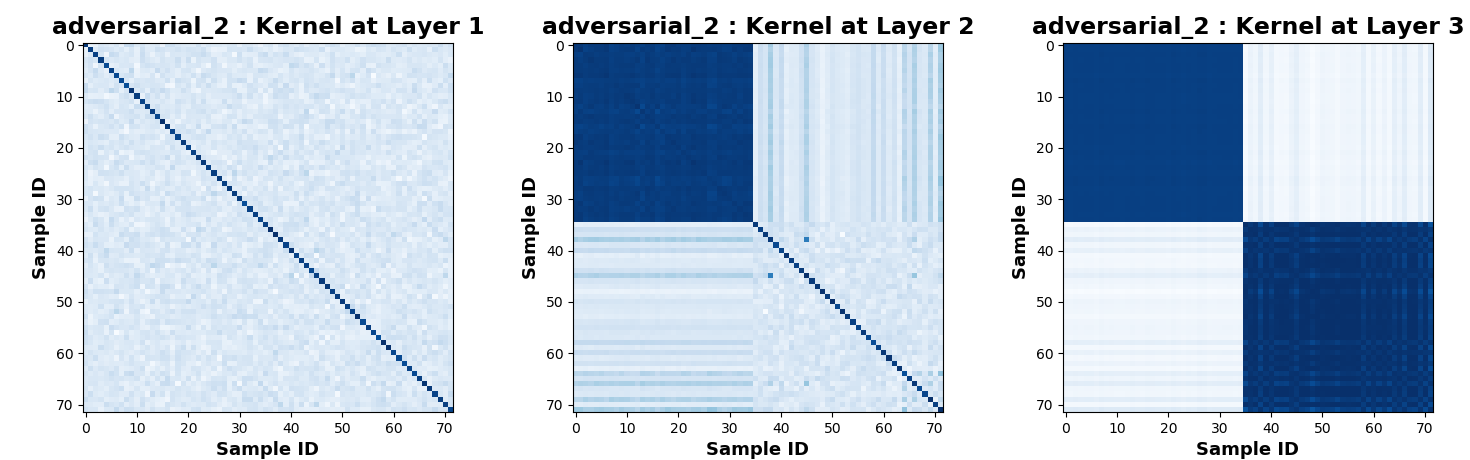}
    \caption{The kernel sequence for the Adversarial dataset.}
\end{figure} 

\begin{figure}[h]
\center
    \includegraphics[width=9cm]{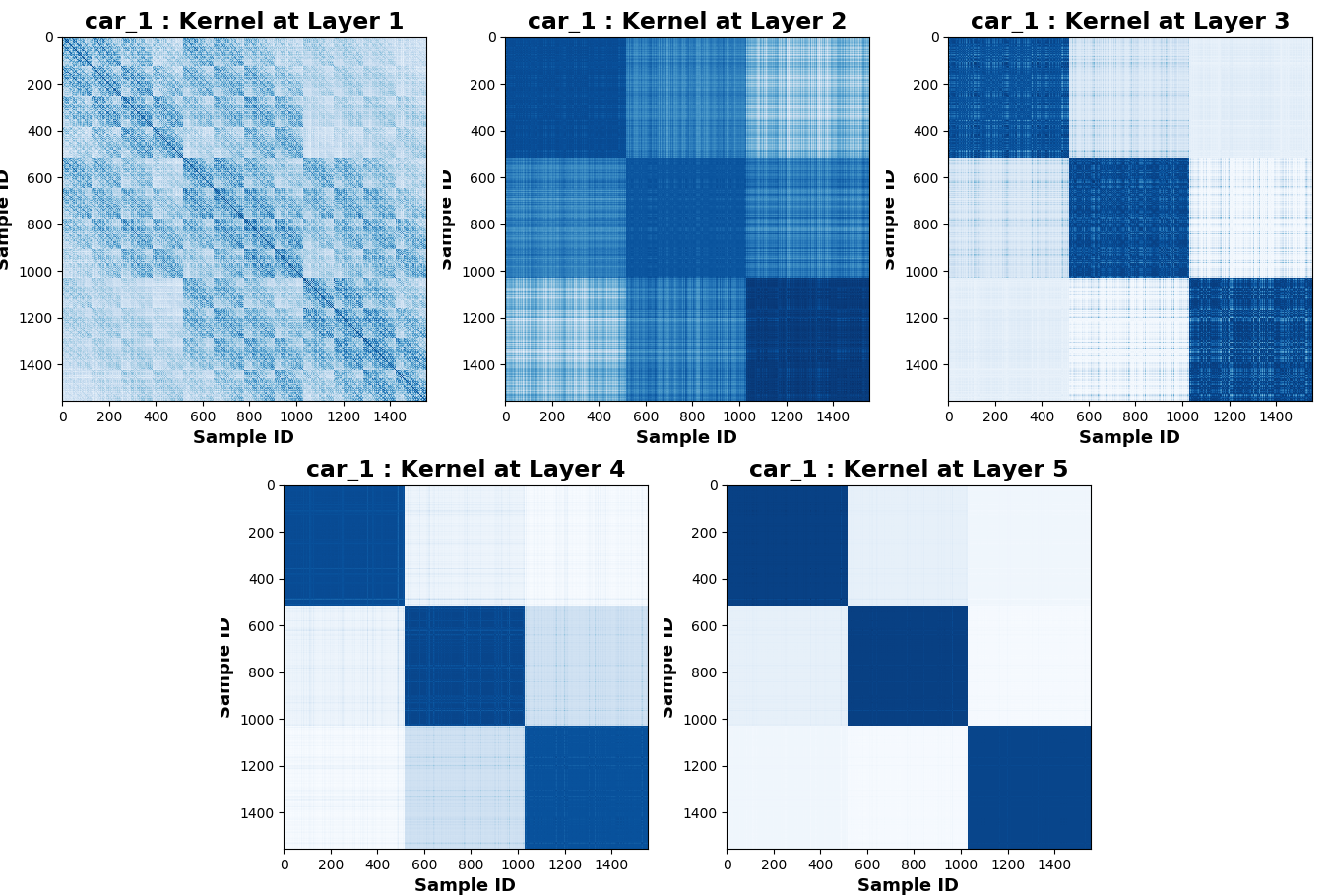}
    \caption{The kernel sequence for the car dataset.}
\end{figure} 

\begin{figure}[h]
\center
    \includegraphics[width=11cm]{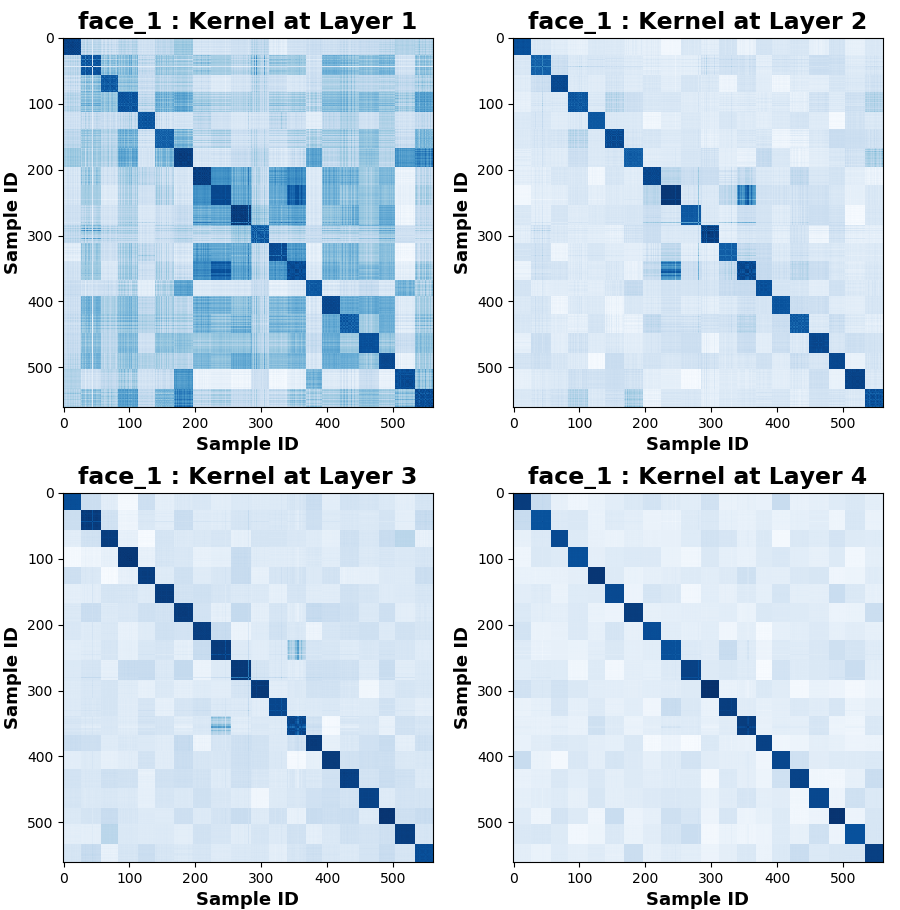}
    \caption{The kernel sequence for the face dataset.}
\end{figure} 

\begin{figure}[h]
\center
    \includegraphics[width=9cm]{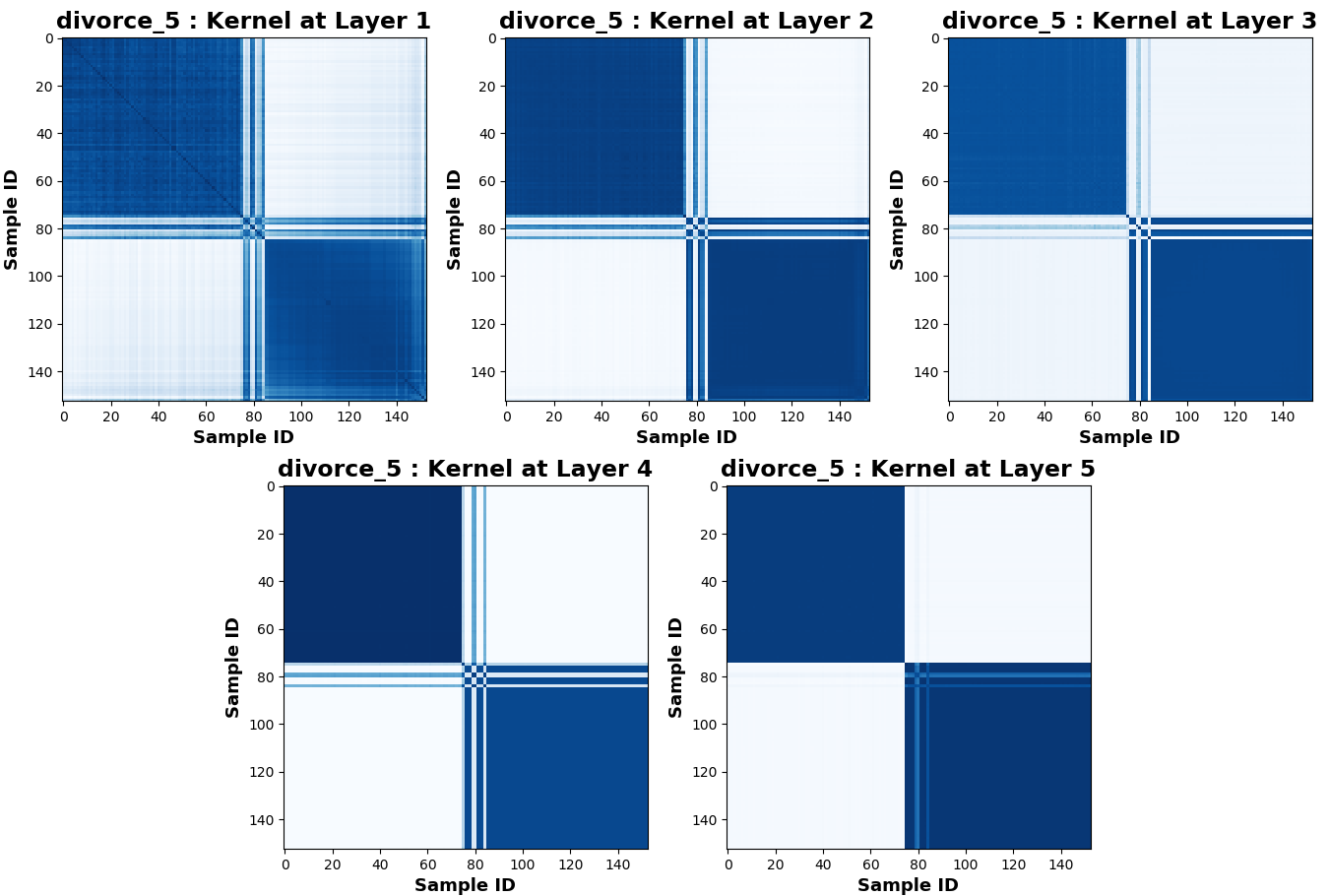}
    \caption{The kernel sequence for the divorce dataset.}
\end{figure} 

\begin{figure}[h]
\center
    \includegraphics[width=11cm]{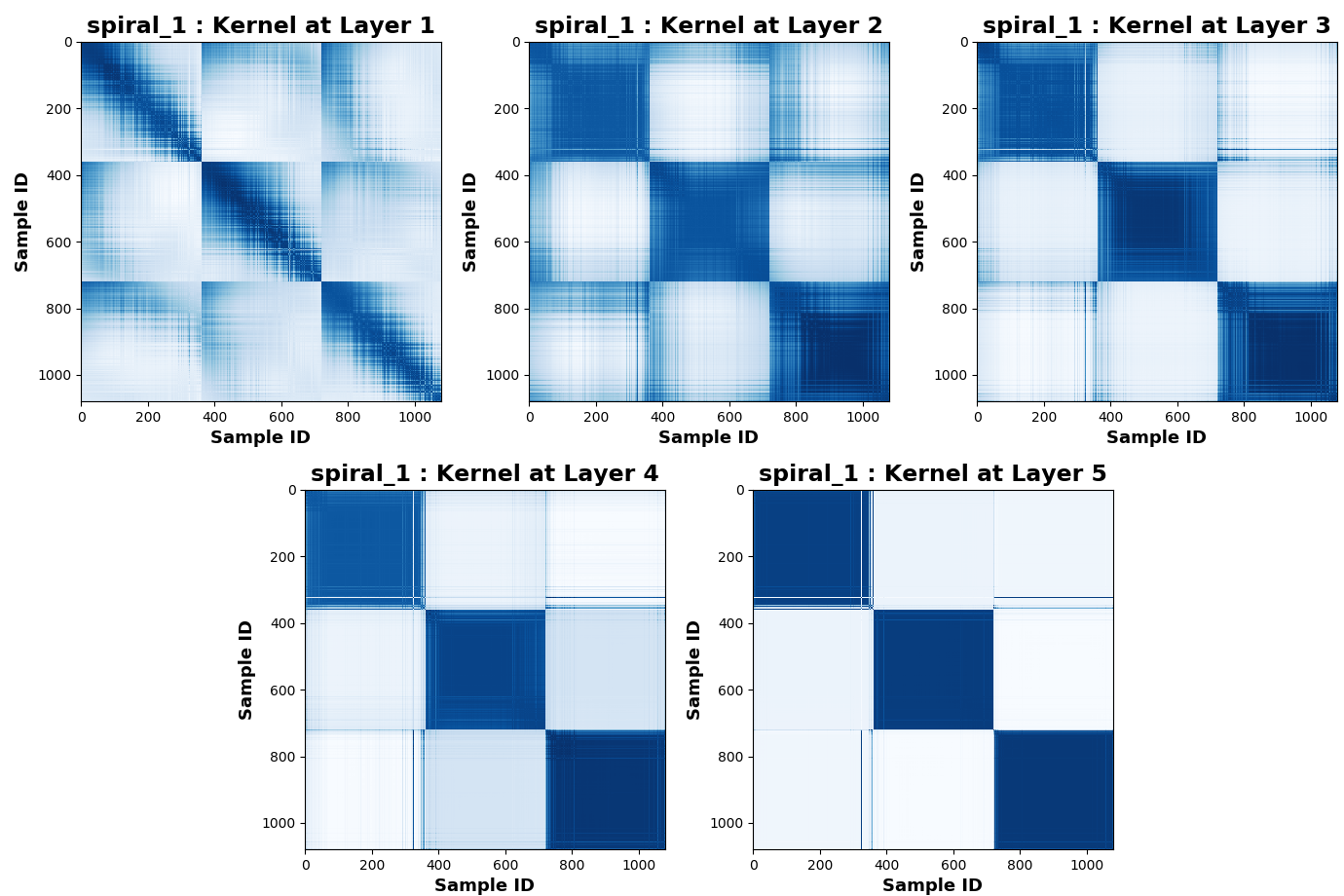}
    \caption{The kernel sequence for the spiral dataset.}
\end{figure}

\begin{figure}[h]
\center
    \includegraphics[width=11cm]{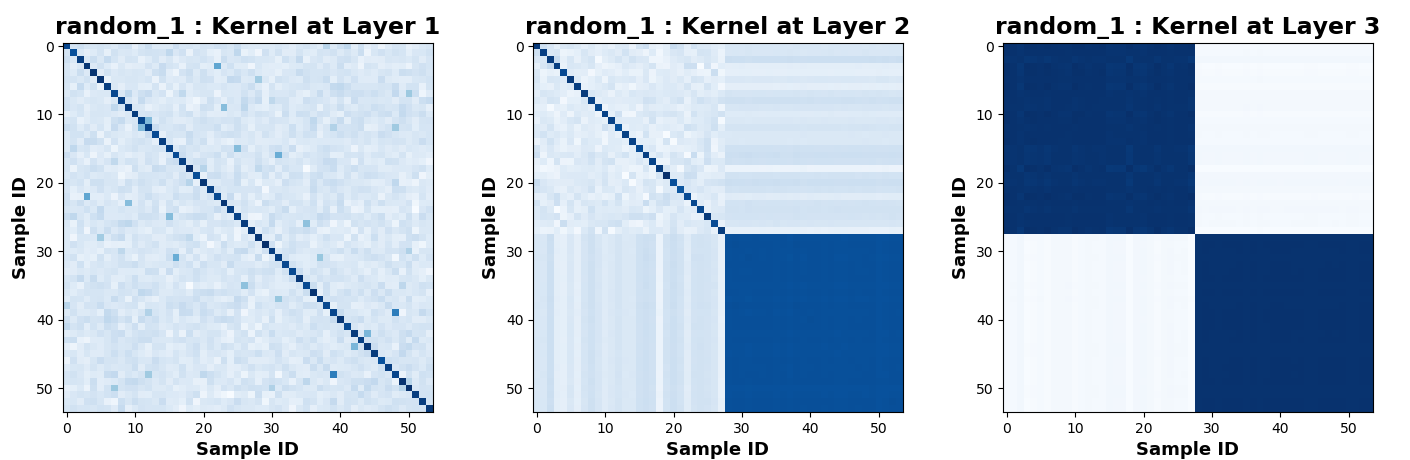}
    \caption{The kernel sequence for the Random dataset.}
\end{figure} 
\end{appendices}
\clearpage
\begin{appendices}
\section{On Generalization. } Besides being an optimum solution, $W_l^*$ exhibits many advantages over $W_s$. For example, while $W_s$ experimentally performs well, $W^*$ converges with fewer layers and superior generalization. This raises a well-known question on generalization. It is known that overparameterized MLPs can generalize even without any explicit regularizer \citep{Zhang2017UnderstandingDL}. This observation contradicts classical learning theory and has been a longstanding puzzle \citep{cao2019generalization,brutzkus2017sgd,allen2019learning}. 
Therefore, by being overparameterized with an infinitely wide network, \kn's ability under HSIC to generalize raises similar questions. In both cases, $W_s$ and $W^*$, the HSIC objective employs an infinitely wide network that should result in overfitting.  We ask theoretically, under our framework, what makes HSIC and $W^*$ special? 

Recently, \citet{poggio2020complexity} have proposed that  traditional MLPs generalize  because gradient methods implicitly regularize the normalized weights given an exponential objective (like our HSIC). We discovered a similar impact the process of finding $W^*$ has on HSIC, i.e., HSIC can be reformulated to isolate out $n$ functions $[D_1(W_l), ..., D_n(W_l)]$ that act as a penalty term during optimization. Let $\cS_i$ be the set of samples that belongs to the $i_{th}$ class and let $\cS^c_i$ be its complement, then each function $D_i(W_l)$ is defined as \begin{equation}
    D_i(W_l) =  
    \frac{1}{\sigma^2}
    \sum_{j \in \cS_i}
    \Gij \kf_{W_l}(r_i,r_j)
    -
    \frac{1}{\sigma^2}
    \sum_{j \in \cS^c_i}
    |\Gij| \kf_{W_l}(r_i,r_j).
\end{equation}
Notice that $D_i(W_l)$ is simply \eq{eq:similarity_hsic} for a single sample scaled by $\frac{1}{\sigma^2}$. Therefore, improving $W_l$ also leads to an increase and decrease of $\kf_{W_l}(r_i,r_j)$ associated with $\cS_i$ and $\cS^c_i$ in \eq{eq:penalty_term}, thereby increasing the size of the penalty term  $D_i(W_l)$.  To appreciate how $D_i(W_l)$ penalizes $\hsic$, we propose an equivalent formulation in the theorem below with its derivation in App~\ref{app:thm:regularizer}.
\begin{theorem}
\eq{eq:similarity_hsic} is equivalent to 
\begin{equation}
    \max_{W_l} 
    \sum_{i,j}
    \frac{\Gij}{\sigma^2}
    \ISMexp
    (r_i^TW_lW_l^Tr_j) 
    -
    \sum_{i}
    D_i(W_l)
    ||W_l^Tr_i||_2.
\end{equation}
\end{theorem}
Based on Thm.~\ref{thm:regularizer}, $D_i(W_l)$ adds a negative variable cost to the sample norm, $||W_l^Tr_i||_2$, prescribing an implicit regularizer on HSIC. As $W_l$ improve HSIC, it also imposes a heavier penalty on \eq{eq:generalization_formulation}, severely constraining $W_l$.

\end{appendices}

\begin{appendices} 
\section{Proof for Theorem \ref{thm:regularizer}}
\label{app:thm:regularizer}

\textbf{Theorem \ref{thm:regularizer}: }
\textit{\eq{eq:similarity_hsic} objective is equivalent to }
\begin{equation}
        \sum_{i,j}
        \Gij \ISMexp
        (r_i^TWW^Tr_j)
        -
        \sum_{i}
        D_i(W)
        ||W^Tr_i||_2.
\end{equation}

\begin{proof}
Let $A_{i,j} = (r_i - r_j)(r_i - r_j)^T$. Given the Lagranian of the HSIC objective as 
\begin{equation}
    \mathcal{L} = -\sum_{i,j} \Gij \ISMexp - \Tr[\Lambda(W^TW - I)].
\end{equation}
Our layer wise HSIC objective becomes 
\begin{equation}
    \min_W -\sum_{i,j} \Gij \ISMexp - \Tr[\Lambda(W^TW - I)].
    \label{obj:hsic}
\end{equation}
We take the derivative of the Lagrangian, the expression becomes
   \begin{equation} \label{eq:gradient_of_lagrangian}
    \nabla_W \mathcal{L} ( W, \Lambda) = \sum_{i, j} \frac{\Gamma_{i,
      j}}{\sigma^2} e^{- \frac{\Tr (W^T A_{i, j} W)}{2 \sigma^2}} A_{i, j} W
      - 2 W \Lambda. 
    \end{equation}
Setting the gradient to 0, and consolidate some scalar values into $\hat{\Gamma}_{i,j}$, we get the expression
   \begin{align} 
    \left[
    \sum_{i, j} \frac{\Gamma_{i,
      j}}{2\sigma^2} e^{- \frac{\Tr (W^T A_{i, j} W)}{2 \sigma^2}} A_{i, j} 
     \right]
      W
      & =  W \Lambda \\
      \left[\frac{1}{2}
      \sum_{i,j} \hat{\Gamma}_{i,j} A_{i,j}
      \right] W
      & = W \Lambda \\
      \mathcal{Q} W
      & = W \Lambda.      
    \end{align}
From here, we see that the optimal solution is an eigenvector of $\mathcal{Q}$. Based on ISM, it further proved that the optimal solution is not just any eigenvector, but the eigenvectors associated with the smallest values of $\mathcal{Q}$. From this logic, ISM solves objective~(\ref{obj:hsic}) with a surrogate objective
\begin{equation}
    \min_W \quad
        \Tr \left( W^T
        \left[\frac{1}{2}
        \sum_{i,j} \hat{\Gamma}_{i,j} A_{i,j}
        \right] W \right) \quad \st W^TW=I.
    \label{obj:min_2}
\end{equation}
Given $D_{\hat{\Gamma}}$ as the degree matrix of $\hat{\Gamma}$ and $R = [r_1, r_2, ...]^T$, ISM further shows that Eq.~(\ref{obj:min_2}) can be written into
\begin{align}
    \min_W \quad
        \Tr \left( W^T R^T
        \left[
        D_{\hat{\Gamma}} - \hat{\Gamma}
        \right] R W \right) &\quad \st W^TW=I \\
     \max_W \quad
        \Tr \left( W^T R^T
        \left[
        \hat{\Gamma} - 
        D_{\hat{\Gamma}}
        \right] R W \right) &\quad \st W^TW=I \\       
     \max_W \quad
        \Tr \left( W^T R^T
        \hat{\Gamma}  
        R W \right) 
        -
        \Tr \left( W^T R^T
        D_{\hat{\Gamma}}  
        R W \right) 
        &\quad \st W^TW=I \\              
     \max_W \quad
        \Tr \left( 
        \hat{\Gamma}  
        R W W^T R^T\right) 
        -
        \Tr \left( 
        D_{\hat{\Gamma}}  
        R W W^T R^T\right) 
        &\quad \st W^TW=I \\                     
     \max_W \quad
        \sum_{i,j}
        \hat{\Gamma}_{i,j}
        [R W W^T R^T]_{i,j}
        -
        \sum_{i,j}
        D_{\hat{\Gamma}_{i,j}}
        [R W W^T R^T]_{i,j}       
        &\quad \st W^TW=I. 
\end{align}
Since the jump from \eq{obj:min_2} can be intimidating for those not familiar with the literature, we included a more detailed derivation in App.~\ref{app:matrix_derivation}.

Note that the degree matrix $D_{\hat{\Gamma}}$ only have non-zero diagonal elements, all of its off diagonal are 0. Given $[RWW^TR^T]_{i,j} = (r_i^TWW^Tr_j)$, the objective becomes
\begin{equation}
     \max_W \quad
        \sum_{i,j}
        \hat{\Gamma}_{i,j}
        (r_i^TWW^Tr_j)
        -
        \sum_{i}
        D_i(W)
        ||W^Tr_i||_2
        \quad \st W^TW=I.     
\end{equation}
Here, we treat $D_{i}$ as a penalty weight on the norm of the $W^Tr_i$ for every sample. 
\end{proof}

To better understand the behavior of $D_i(W)$, note that $\hat{\Gamma}$ matrix looks like
\begin{equation}
    \hat{\Gamma} = \frac{1}{\sigma^2} \begin{bmatrix}
       \begin{bmatrix}
       \Gamma_{\mathcal{S}} \ISMexp
        \end{bmatrix}
       &  
       \begin{bmatrix}
       -|\Gamma_{\mathcal{S}^c}| \ISMexp
       \end{bmatrix}
       &
       ... \\
       \begin{bmatrix}
       -|\Gamma_{\mathcal{S}^c}| \ISMexp
       \end{bmatrix}
       &      
       \begin{bmatrix}
       \Gamma_{\mathcal{S}} \ISMexp
        \end{bmatrix}      
       &            
       ... \\
       ... &
       ... &
       ...
    \end{bmatrix}.
\end{equation}
The diagonal block matrix all $\Gij$ elements that belong to $\cS$ and the off diagonal are elements that belongs to $\cS^c$. Each penalty term is the summation of its corresponding row. Hence, we can write out the penalty term as
\begin{equation}
    D_i(W_l) = 
    \frac{1}{\sigma^2}
    \sum_{j \in \cS|i}
    \Gij \kf_{W_l}(r_i,r_j)
    -
    \frac{1}{\sigma^2}
    \sum_{j \in \cS^c|i}
    |\Gij| \kf_{W_l}(r_i,r_j).
\end{equation}
From this, it shows that as $W$ improve the objective, the penalty term is also increased. In fact, at its extreme as $\hsic_l \rightarrow \hsic^*$, all the negative terms are gone and all of its positive terms are maximized and this matrix approaches 
\begin{equation}
    \hat{\Gamma}^* = \frac{1}{\sigma^2} \begin{bmatrix}
       \begin{bmatrix}
       \Gamma_{\mathcal{S}} 
        \end{bmatrix}
       &  
       \begin{bmatrix}
       0
       \end{bmatrix}
       &
       ... \\
       \begin{bmatrix}
       0
       \end{bmatrix}
       &      
       \begin{bmatrix}
       \Gamma_{\mathcal{S}} 
        \end{bmatrix}      
       &            
       ... \\
       ... &
       ... &
       ...
    \end{bmatrix}.
\end{equation}
From the matrix $\hat{\Gamma}^*$ and the definition of $D_i(W_l)$, we see that as $\mathcal{K}_W$ from $\mathcal{S}$ increase, 

Since $D_i(W)$ is the degree matrix of $\hat{\Gamma}$, we see that as $\hsic_l \rightarrow \hsic^*$, we have
\begin{equation}
   D^*_i(W) > D_i(W). 
\end{equation}
\end{appendices} 
\begin{appendices} 
\section{Derivation for \texorpdfstring{$\sum_{i,j} \Psi_{i,j} (x_i - x_j)(x_i - x_j)^T = 2X^T(D_\Psi - \Psi)X$ }\xspace }
\label{app:matrix_derivation}
Since $\Psi$ is a symmetric matrix, and $A_{i, j} = ( x_i - x_j) ( x_i -
x_j)^T $, we can
rewrite the expression into
\[ \begin{array}{lll}
     \sum_{i, j} \Psi_{i, j} A_{i, j}  & = &  \sum_{i, j} \Psi_{i, j} ( x_i -
     x_j) ( x_i - x_j)^T\\
     & = & \sum_{i, j} \Psi_{i, j} ( x_i x_i^T - x_j x_i^T - x_i x_j^T +
     x_j x_j^T)\\
     & = & 2 \sum_{i, j} \Psi_{i, j} ( x_i x_i^T - x_j x_i^T)\\
     & = & \left[ 2 \sum_{i, j} \Psi_{i, j} ( x_i x_i^T) \right] - \left[ 2
     \sum_{i, j} \Psi_{i, j} ( x_i x_j^T) \right] .
   \end{array} \]
If we expand the 1st term, we get
\begin{align}
    2\sum_{i}^n \sum_{j}^n 
    \Psi_{i, j} ( x_i x_i^T)
    & = 
    2 \sum_i
   \Psi_{i, 1} ( x_i x_i^T) +
   \ldots + \Psi_{i, n} ( x_i x_i^T) \\
   &= 
   2 \sum_{i}^n 
   [\Psi_{1, 1} + \Psi_{1, 2} + ...]
   x_i x_i^T \\
   &= 
   2\sum_{i}^n 
   d_i 
   x_i x_i^T \\  
   &=
   2 X^TD_\Psi X
\end{align}
Given $\Psi_i$ as the $i$th row, next we look at the 2nd term
\begin{align}
    2 \sum_i \sum_j \Psi_{i, j}  x_i x_j^T
    &=
    2 \sum_i \Psi_{i, 1}  x_i x_1^T
    + \Psi_{i, 2}  x_i x_2^T
    + \Psi_{i, 3}  x_i x_3^T + ...\\
    &=
    2 \sum_i   x_i (\Psi_{i, 1} x_1^T)
    + x_i (\Psi_{i, 2} x_2^T)
    + x_i (\Psi_{i, 3} x_3^T) + ...\\
    &=
    2 \sum_i   x_i 
    \left[
    (\Psi_{i, 1} x_1^T)
    + (\Psi_{i, 2} x_2^T)
    + (\Psi_{i, 3} x_3^T) + ...
    \right]\\
    &=
    2 \sum_i   x_i 
    \left[
    X^T \Psi_i^T
    \right]^T\\   
    &=
    2 \sum_i   x_i 
    \left[
    \Psi_i X
    \right]\\      
    &=
    2 \left[   
    x_1 \Psi_1 X + 
    x_2 \Psi_2 X + 
    x_3 \Psi_3 X + ...
    \right]\\         
    &=
    2 \left[   
    x_1 \Psi_1  + 
    x_2 \Psi_2  + 
    x_3 \Psi_3  + ...
    \right]X \\            
    &=
    2X^T \Psi  X \\               
\end{align}
Putting both terms together, we get
\begin{align}
    \sum_{i,j} \Psi_{i,j} A_{i,j} &=  2 X^TD_\Psi X -  2X^T \Psi  X a\\
    &=  2 X^T[ D_\Psi - \Psi] X \\
\end{align}

\end{appendices}

%

\end{document}